\documentclass[twoside]{article}
\usepackage[accepted]{arxiv}
\usepackage{url}

\usepackage[utf8]{inputenc} 
\usepackage[T1]{fontenc}    
\usepackage{hyperref}       
\usepackage{url}            
\usepackage{booktabs}       
\usepackage{amsfonts}       
\usepackage{nicefrac}       
\usepackage{microtype}      
\usepackage{amsmath,mathrsfs,amssymb,amsthm}
\usepackage{algpseudocode,algorithm,algorithmicx}
\usepackage{graphicx}
\usepackage{rotating}
\usepackage{balance}
\usepackage[round]{natbib}
\usepackage{hyperref}
\usepackage{epstopdf}
\hypersetup{
colorlinks=true,
    citecolor = {blue}
}

\newcommand{\T}{{\rm T}}
\renewcommand{\H}{{\rm H}}
\newcommand{\bz}{{\bf Z}}
\newcommand{\bw}{{\bf W}}
\newcommand{\wtd}{\widetilde}
\newcommand{\wht}{\widehat}
\newcommand{\R}{\mathbb{R}}
\newcommand{\I}{\mathcal{I}}

\newcommand{\PO}{\Pi_{\Omega}}
\newcommand{\PT}[1]{\Pi_{\Omega_{#1}}}

\newcommand{\POT}{\mathscr{P}_{\Omega_t}}

\DeclareMathOperator{\tol}{tol}

\DeclareMathOperator{\opt}{opt}

\DeclareMathOperator{\rank}{rank}
\DeclareMathOperator{\diag}{diag}

\DeclareMathOperator{\argmin}{argmin}
\DeclareMathOperator{\tr}{trace}
\DeclareMathOperator{\supp}{supp}
\DeclareMathOperator{\subspan}{span}

 \newtheorem{theorem}{Theorem}
 \newtheorem{lemma}{Lemma}

 \newtheorem{definition}{Definition}
 \newtheorem{remark}{Remark}
 \newtheorem{example}{Example}

\begin{document}

%

%

\runningtitle{Solving the Robust Matrix Completion Problem via  a System of Nonlinear Equations}

\twocolumn[

\aistatstitle{\vspace{-0in}Solving the Robust Matrix Completion Problem via \\ a System of Nonlinear Equations\vspace{-0in}}

\aistatsauthor{ Yunfeng Cai and Ping Li }

\aistatsaddress{  Cognitive Computing Lab\\
Baidu Research\\
No. 10 Xibeiwang East Road, Beijing 100085, China\\
10900 NE 8th St. Bellevue, WA 98004, USA\\ \{caiyunfeng, liping11\}@baidu.com}]

\begin{abstract}\vspace{-0.1in}
We consider the problem of robust matrix completion,
which aims to recover a low rank matrix $L_*$ and a sparse matrix $S_*$ from incomplete observations of their sum $M=L_*+S_*\in\mathbb{R}^{m\times n}$.
Algorithmically, the robust matrix completion problem is transformed into a problem of solving a system of nonlinear equations,
and the alternative direction method is then used to solve the nonlinear equations.
In addition, the algorithm is highly parallelizable and suitable for large scale problems.
Theoretically, we characterize the sufficient conditions for when $L_*$ can be approximated by a low rank approximation of the observed $M_*$.
And under proper assumptions, it is shown that the algorithm converges to the true solution linearly.
Numerical simulations show that the simple method works as expected and is comparable with state-of-the-art methods.
\end{abstract}

\vspace{-0.2in}
\section{Introduction}

{\em Robust matrix completion} (RMC) \citep{Proc:Chen_ICML11,  tao2011recovering, Proc:Cherapanamjeri_ICML17, klopp2017robust, Article:Zeng_TSP18}  aims to recover a low rank matrix $L_*$ and a sparse matrix $S_*$ from a sampling of $M=L_*+S_*$.
Mathematically, RMC can be formulated as the following optimization problem \citep{Article:Candes_JACM11,chandrasekaran2011rank}:
\begin{align*}
\textrm{RMC}:\qquad  \min_{L, S} \rank(L) + \lambda |\supp(S)|,\\
 \textrm{s.t.}\quad L_{ij}+S_{ij}=M_{i,j}, \quad (i,j)\in\Omega,
\end{align*}
where $\lambda$ is a tuning parameter, $\Omega$ is a subset of $\{1,\dots,m\}\times\{1,\dots,n\}$.
When $S=0$, RMC becomes the matrix completion (MC) problem~\citep{candes2009exact,Proc:Meka_NIPS09,cai2010singular, Article:Candes_IT10, Proc:Jain_COLT15, Article:Liu_TSP16};
When $\Omega=\{1,\dots,m\}\times\{1,\dots,n\}$, RMC becomes the robust principal component analysis (RPCA) \citep{jolliffe2011principal}.
Thus, RMC can be taken as a combination/generalization of MC and RPCA \citep{Article:Candes_IEEE10,Proc:Jain_COLT15,Proc:Jain_STOC13,Article:Keshavan_IT10}.

In many scientific and engineering problems, people need to recover a low rank matrix from observed data,
e.g., the recommender system~\citep{funk2006netflix,Article:Candes_IEEE10,Proc:Hu_WWW17,Proc:Hu_WSDM18,Proc:Hu_CIKM18}, social network analysis~\citep{Proc:Huang_IJCAI13},
machine learning~\citep{Article:Candes_IEEE10,Article:Davenport_2016},  image impainting~\citep{Proc:Bertalmio_SIGGRAPH00},
computer vision~\citep{Article:Candes_IEEE10}, bioinformatics~\citep{Article:Kim_BioInf05}, etc.

MC/RPCA/RMC has been studied extensively from an optimization point of view,
many algorithms are proposed, and exact recovery is discussed under proper assumptions.
Most of well-known algorithms are based on convex optimization,
in which the rank of a matrix is relaxed to its nuclear norm (the sum of all singular values),
and the number of nonzero entries of a matrix is relaxed to its $\ell_1$-norm (the sum of absolute values of all entries),
e.g.,~\citep{cai2010singular,Article:Candes_JACM11,candes2009exact,Article:Candes_IT10,recht2010guaranteed,klopp2017robust}.
However, the computation of the nuclear norm of $L$, which requires the computation of its singular value decomposition (SVD),
is expensive and unsuitable for parallelization, as a result, algorithms based on the nuclear norm relaxation are often not very realistic for large matrices.

To deal with large problems, the low rank matrix can be represented as the product of a tall-skinny matrix and a short-fat matrix,
 so that the low rank property is satisfied automatically. But, the optimization problem becomes nonconvex, which makes it difficult to solve.
Bi-convex as the problem is, the alternative minimization can be used to solve it more efficiently, e.g.,~\citep{Proc:Jain_STOC13}.
In~\citep{Proc:Yi_NIPS16}, gradient descent method is used to solve RMC, which is shown to be fast.
In~\citep{Proc:Cherapanamjeri_ICML17}, projected gradient method with hard-thresholding is used to solve RMC, with nearly-optimal observation and corruption.
We refer the readers to~\citep{Article:Zeng_TSP18} and reference therein for more methods.
Besides the optimization based method, quite recently, in~\citep{Proc:Dutta_AAAI19}, the RPCA/RMC is solved via an alternating nonconvex projection method. This method does not require any objective function, convex relaxation or surrogate convex constraint.

\textbf{Contribution.} \ In this paper, we solve the RMC problem via solving a system of nonlinear equations (NLEQ).
This method does not require any objective function, convex relaxation or surrogate convex constraint, either.
Let $L_*=XY^{\T}$, where $X=[x_1,\dots, x_m]^{\T}\in\R^{m\times r}$, $Y=[y_1,\dots, y_n]^{\T}\in\R^{n\times r}$.
The RMC problem can be formulated as the following system of NLEQ:
\begin{equation}\label{eq:rmc}
x_i^{\T}y_j=M_{ij}, \ \mbox{for } (i,j)\in\Omega\setminus\supp(S_*),
\end{equation}
where $\supp(S_*)$ is the support set of $S_*$.
When the rank $r$ and the support set of $S_*$ are both known, solving RMC amounts solving a system of nonlinear equations.
Any numerical method for nonlinear equations can be used to solve it (e.g., the steepest descent method, the Newton method, etc.),
among which the simplest one is the alternative direction method (ADM):
Fixing $X$ (or $Y$), $Y$ (or $X$) can be updated via solving an (overdetermined) linear system of equations (usually in least square sense).
Thus, solving RMC via solving \eqref{eq:rmc} heavily depends on whether we can solving \eqref{eq:rmc} without knowing $r$ and $\supp(S_*)$.
It is worth mentioning here that in~\citep{Proc:Meka_NIPS09}, such a method was proposed to solve the MC problem.
And it was stated there that `` ... its variants outperform most methods in practice. However, analyzing the performance of alternate minimization is a notoriously hard problem.''

The contributions of this paper are three folds.
First, for both full and partial observation cases, we characterize some sufficient conditions for when the low rank approximation of the observed $M$ is approximate $L_*$.
Second, we propose to solve RMC via solving a system of NLEQ rather than optimization,
and an ADM method, which carefully handles the unknown $r$ and $\supp(S_*)$ issue, is developed.
Third, we carefully analyze the convergence of the ADM,
and it is shown that under proper assumptions, the ADM converges to the true solution linearly, i.e.,  exact recovery can be achieved.
So, we give an answer to a problem which is even more difficult than the aforementioned ``notoriously hard problem''.
In addition, the algorithm is highly parallelizable and naturally suitable for large scale problems.
It is also worth mentioning here that the results of this paper are applicable to the MC problem as well as the RPCA problem.

The rest of this paper is organized as follows.
In Section~\ref{sec:alg}, we first develop the algorithm, followed by its convergence analysis in Section~\ref{sec:convergence}.
Numerical experiments are presented in Section~\ref{sec:numer}.
Concluding remarks are given in Section~\ref{sec:conclusion}.

\noindent{\bf Notation.}
We shall adopt the MATLAB style convention to access the entries of vectors and matrices.
The set of integers from $i$ to $j$ inclusive is $i : j$.
For a matrix $A$, its submatrices $A_{(k:\ell,i:j)}$, $A_{(k:\ell,:)}$, $A_{(:,i:j)}$ consist of intersections of
row $k$ to row $\ell$ and column $i$ to column $j$,
row $k$ to row $\ell$ and all columns,
all rows and column $i$ to column $j$, respectively.
$A_{(j,:)}$ and $A_{(:,k)}$ denote the $j$th row and $k$th column of $A$, respectively.
 $\|A\|$ stands for the spectral norm of $A$,
$\|A\|_F$ denotes the Frobenius norm,
$\|A\|_1=\sum_{i,j}|a_{ij}|$, $\|A\|_{\max}=\max_{i,j}|a_{ij}|$, $\|A\|_{2,\infty}= \max_i\|A_{(i,:)}\|$, $A^{\dagger}$ stands for the Moore--Penrose inverse,
and $\kappa(A)=\|A\|\|A^{\dagger}\|$ denote the condition number of $A$.
Denote by $\sigma_j(A)$ for $1\le j\le \min\{m,n\}$
the singular values of $A$ and
they are always arranged in a non-increasing order:
$\sigma_1(A)\ge\sigma_2(A)\ge\dots\ge\sigma_{\min\{m,n\}}(A)$.
$\mathcal{R}(A)$ stands for the range space of $A$, i.e., $\mathcal{R}(A)=\subspan\{y\in\R^m\; | \; y=Ax, x\in\R^n\}$.
The vector $e_j$ stands for the $j$th column of the identity matrix $I$.
Furthermore,
for an index set $\Omega\subset\{1,\dots, m\}\times \{1,\dots,n\}$, $|\Omega|$ denotes the cardinality of $\Omega$,
$\PO(A)=[{\I}_{\{(i,j)\in\Omega\}} a_{ij}]\in\R^{m\times n}$, where ${\I}$ is the indicator function.

\vspace{-0.07in}
\section{Algorithm}\label{sec:alg}
\vspace{-0.07in}

In this section, we first reformulate the RMC problem as a problem of solving a system of nonlinear equations (NLEQ),
show how to solve the NLEQ via an alternative direction method (ADM), then the overall algorithm is summarized.

\vspace{-0.05in}
\subsection{Problem Reformulation, Difficulty and Solution}\label{sec:adm}

Let $M=XY^{\T}$, where $X=[x_1,\dots, x_m]^{\T}\in\R^{m\times r}$, $Y=[y_1,\dots, y_n]^{\T}\in\R^{n\times r}$.
The MC problem can be formulated as
\begin{equation}\label{eq:mc}
x_i^{\T}y_j=M_{ij}, \ \mbox{for } (i,j)\in\Omega.
\end{equation}
Now recall \eqref{eq:rmc},
the task of RMC becomes solving \eqref{eq:mc} with unknown $r$ and minimum violators.
Here by  a ``violator'', denoted by $(i',j')$, we mean that $x_{i'}^{\T}y_{j'} \ne M_{i'j'}$, it will be also referred to as  an ``outlier'' hereafter.

The ADM is the simplest method to solve the NLEQ of form \eqref{eq:mc} -- given an initial guess for $X$,
we fix $X$,  \eqref{eq:mc} becomes a linear system in $Y$ (which is assumed to be overdetermined),
 we can solve the linear system for $Y$;  similarly, we fix $Y$ to solve $X$; the iteration continues until convergence.
As $r$ and $\supp(S_*)$ are unknown, our task is to determine them during the iteration of ADM.
Next, we first show how to determine $r$, then $\supp(S_*)$.

\vspace{-0.1in}
\paragraph{Determine the rank adaptively}
In the $t$th iteration of ADM,
let the current rank estimation be $r_t$,
$Y_t$ be current estimation for $Y$ and
$Y_t$ have orthonormal columns, i.e., $Y_t^{\T}Y_t=I_{r_t}$.
By ADM, the estimation of $X$ can be obtained, denote it by $\wtd{X}_{t+1}$.
In order to update $r_t$, we need to compute the singular values of  $\wtd{X}_{t+1}Y_t^{\T}$. Noticing that $Y_t$ is orthonormal, then the top $r_t$ singular values of $\wtd{X}_{t+1}Y_t^{\T}$ will be the singular values of $\wtd{X}_{t+1}$.
Then the singular values can be obtained as follows:
first, compute the QR decomposition of $\wtd{X}_{t+1}$:
\begin{equation*}\label{qrx}
\wtd{X}_{t+1}=\wht{X}_{t+1}R_{x,t+1},
\end{equation*}
where $\wht{X}_{t+1}\in\R^{m\times r_t}$ is orthonormal, $R_{x,t+1}\in\R^{r_t\times r_t}$;
second, compute the SVD of $R_{x,t+1}$:
\begin{equation*}\label{svdyr}
R_{x,t+1}=Q_x\Sigma Q_y^{\T},
\end{equation*}
where $Q_x\in\R^{r_t\times r_t}$, $Q_y\in\R^{n\times r_t}$ are orthogonal,
$\Sigma=\diag(\hat{\sigma}_1,\dots,\hat{\sigma}_{r_t})$ with $\hat{\sigma}_1\ge \dots\ge \hat{\sigma}_{r_t}\ge 0$.
Then the singular values of $\wtd{X}_{t+1}Y_t^{\T}$ are $\hat{\sigma}_1,\dots, \hat{\sigma}_{r_t}$.
Similarly, when $X_t$ is the current estimation for $X$, and $X_t^{\T}X_t=I_{r_t}$,
we can compute an estimation for $Y$, its singular values can be obtained.

Let $L_t$ be the current estimation for $L_*$.
When the rank is underestimated, i.e., $r_t< r$, the residual $\tau_t=\|\PO(L_t+S_t-M)\|_F$ will stagnate,
in such case, we increase the estimated rank $r_t$.
When the rank is overestimated, i.e., $r_t>r$,
we expect to observe rank deficiency from the singular values $\hat{\sigma}_1,\dots,\hat{\sigma}_{r_t}$.
In such a case, we decrease the estimated rank $r_t$.
For the RMC problem, we prefer an overestimated rank over an underestimated rank due to the following reason.
The residual $\tau_t$ stagnates for two reasons:
one is that the estimated rank is smaller than the true rank;
the other is that $|\supp(S_*)\setminus\supp(S_t)|$ is large.
Then when the residual stagnates, it is difficult for us to make a good choice --
 to increase the estimated rank or to drop some equalities (of course, those equalities need to be carefully selected) in \eqref{eq:mc}.
Increasing the estimated rank when $|\supp(S_*)\setminus\supp(S_t)|$ is large
or dropping equalities in \eqref{eq:mc} when the rank is underestimated
will both lead to catastrophic consequences,
such as the estimated rank exceeds a prescribed limit,
too many ``correct'' equalities are dropped which will probably result in underdetermined linear systems when updating $X$ (or $Y$).
With an overestimated rank, when the residual stagnates, we decrease the estimated rank via the singular values of $L_t$;
if there is no rank deficiency in $L_t$, we drop some equalities in  \eqref{eq:mc}.

When an overestimated rank decreases to the actual rank, it is expected that the estimated rank will remain unchanged in the follow-up iterations.
Therefore, we do not need to check the singular values of $L_t$ in each iteration for the sake of efficiency.

\vspace{-0.11in}
\paragraph{Determine $\mathbf{\supp(S_*)}$ via outlier detection}
When a good approximation $\wht{L}$ of $L_*$ is obtained,
 $S_*=M-L_*\approx M-\wht{L}$.
 Thus, it is reasonable to detect $(i,j)\in\Omega\cap\supp(S_*)$
from the residual $\{R_{ij}=M_{ij} - \wht{L}_{ij}\}_{(i,j)\in\Omega}$.

Outlier detection has been used for centuries to remove abnormal data.
Various outlier detection techniques have been used~\citep{Proc:Ester_KDD96,hodge2004survey,Proc:Xu_NIPS10,Proc:Rahmani_NeurIPS19,Proc:Slawski_UAI19}.
In our implementation, we simply determine the outliers as follows:
find the top-$k$ values in each row and column of $|R|$ (unavailable entries of $|R|$ are set to zero),
and the entries in the intersection are taken as outliers.
Alternatively, simply find the top $k'$ values among all entries of $|R|$.
Here $k,k'$ are two parameters which can be tuned.
In what follows, we denote
\begin{align}
\mathcal{T}_s(A)=[b_{ij}],
\end{align}
where $s$ is the number of the removed outliers, $b_{ij}=A_{(i,j)}$ if $A_{(i,j)}$ is an outlier, $b_{ij}=0$, otherwise.
Of course, one can also try other outlier detection techniques.

\vspace{-0.05in}
\subsection{Algorithm details}

Now we present Algorithm~\ref{alg:rmc}, which summarizes the ADM for RMC described in the previous subsection.

\begin{algorithm}[ht]
 \caption{\textsc{ADM for RMC via NLEQ}\label{alg:rmc}}
  \begin{algorithmic}[1]
   \Require{The observed matrix $\PO(M)$,  a sparsity level parameter $s$, an estimated rank $r_0$,
   an upper bound $\kappa$ for the condition number of $L_*$, and a tolerance $\tol$.}
   \Ensure{$X\in\R^{m\times r_t}$, $Y\in\R^{n\times r_t}$ and $S\in\R^{m\times n}$ such that $\|\PO(XY^{\T}+S-M)\|_F\le \tol$, $\|S\|_0\le s$. }

\State Set $S_0=\mathcal{T}_{s}(M)$, $X_0=0$, $Y_0=0$, $\Sigma_0=0$, t=1;
\State Compute $[X_1,\Sigma_1,Y_1]=\mbox{SVD}_{r_0}((M-S_0)/p')$, where $p'=(|\Omega|-s)/mn$;
\State Compute $R_t=\PO(M - X_t\Sigma_tY_t^{\T})$;
\State Set $S_t = \mathcal{T}_{s}(R_t)$, $\Omega_t= \Omega\setminus\supp(S_t)$;
\State Compute $\tau_t = \| \PO(M-X_t\Sigma_tY_t^{\T}-S_t)\|_F$;

\While{${\tau_t} > \tol$}
\State Set $t=t+1$;  
\State Solve $\PT{t-1}(\wtd{X}_tY_{t-1}^{\T})=\PT{t-1}(M)$ for $\wtd{X}_t$;
\State Compute the QR decomposition $\wtd{X}_t=\wht{X}_tR_{x,t}$,
where $\wht{X}_t$ has orthonormal columns, $R_{x,t}$ is upper triangular;

\State Compute the SVD $R_{x,t}=Q_x\wht{\Sigma}Q_y^{\T}$,
where $\wht{\Sigma}=\diag(\hat{\sigma}_1,\dots,\hat{\sigma}_{r_{t-1}})$,
$Q_x, Q_y$ are orthogonal;
\State Set $r_{t}=r_{t-1} - | \{j \; | \; \kappa \, \hat{\sigma}_j< \hat{\sigma}_1\}|$;
\State Set $\wht{X}_t=[\wht{X}_tQ_x]_{(:,1:r_t)}$; 

\State Solve $\PT{t-1}(\wht{X}_t \wtd{Y}_t^{\T})=\PT{t-1}(M)$ for $\wtd{Y}_t$;
\State Compute the QR decomposition $\wtd{Y}_t=\wht{Y}_tR_{y,t}$,
where $Y_t$ has orthonormal columns, $R_{y,t}$ is upper triangular;
\State Compute the SVD $R_{y,t}^{\T}=Q_x\wht{\Sigma}Q_y^{\T}$,
where $\wht{\Sigma}=\diag(\hat{\sigma}_1,\dots,\hat{\sigma}_{r_{t}})$,
$Q_x, Q_y$ are orthogonal;
\State Set $r_t=r_{t} - | \{j \; | \; \kappa \, \hat{\sigma}_j< \hat{\sigma}_1\}|$;
\State Set $X_t=[\wht{X}_tQ_x]_{(:,1:r_t)}$, $Y_t=[\wht{Y}_tQ_y]_{(:,1:r_t)}$, $\Sigma_t=\wht{\Sigma}_{(1:r_t,1:r_t)}$;
\State Compute $R_t=\PO(M - X_t\Sigma_tY_t^{\T})$;
\State Set $S_t = \mathcal{T}_{s}(R_t)$, $\Omega_t= \Omega\setminus\supp(S_t)$;
\State Compute $\tau_t = \| \PO(M-X_t\Sigma_tY_t^{\T}-S_t)\|_F$;
\EndWhile
 \end{algorithmic}
\end{algorithm}

Some implementation details follows.

\vspace{-0.1in}
\paragraph{Initializing $Y_0$}
According to Theorem~\ref{thm2} below, good initial guesses for $X$ and $Y$ can be obtained by computing the SVD of  $\PT{0}(M)$.
An iterative procedure (e.g., Krylov subspace method) is usually adopted to accomplish the task,
in which matrix vector products $\PT{0}(M) v$ and $\PT{0}(M)^{\T} v$ are called several times.
A simpler way, which is more efficient and numerically proven to be reliable, is the following:
compute $W=\PT{0}(M)^{\T}\PT{0}(M)N$, compute an orthonormal basis for $W$, and set the columns of $Y_1$ as the basis.
Here $N\in\R^{n\times r_0}$ is a random matrix with entries drawn from the standard normal distribution.
Such a procedure is essentially one iteration of the subspace method (a generalization of power method to compute several dominant eigenvectors).
Since an initial guess for $X$ or $Y$ is sufficient for ADM to run in Algorithm~\ref{alg:rmc}, it is indeed unnecessary to compute the estimations for both $X$ and $Y$.

\vspace{-0.1in}
\paragraph{Solving $X_t$ and  $Y_t$}
On Lines 8 and 13, $\wtd{X}_t$ and $\wtd{Y}_t$ can both be solved row by row or simultaneously.
And to obtain one row of $\wtd{X}_t$ or $\wtd{Y}_t$, a small linear system needs to be solved.
When the linear system is underdetermined, Algorithm~\ref{alg:rmc} may break down.
Therefore, in each row and column, Algorithm~\ref{alg:rmc} requires the number of observed entries (after the removal of the corrupted entries) must be larger than the rank.
To be more precise, we need the small linear system to be good conditioned.
In general, it is difficult to determine how many rows/columns are needed to ensure the linear system to be good conditioned.
Numerically, for a random matrix $A\in\R^{s\times r}$ (generated from a standard normal distribution) with $s=\mathcal{O}(r)>2r$ is usually good conditioned.
So, we may declare that $\mathcal{O}(r)>2r$ observations in each row and column are sufficient.

In our implementation,  the linear systems are solved in the least square sense.
One may also choose to minimize $\ell_p$-norm ($p\ge 0$) of the residual as in ~\citep{Article:Zeng_TSP18}.

\vspace{-0.1in}
\paragraph{Computational complexity}
When the number of observations in each row and column is $\mathcal{O}(r)$, each linear system can be solved in $\mathcal{O}(r^3)$ FLOPS.
So, in each iteration, the computational complexity of the linear system solving on Lines 8 and 13 is $\mathcal{O}((m+n)r^3)$.
The computational  complexity of the QR decomposition on lines 9 and 14 is $\mathcal{O}((m+n)r^2)$.
The computational complexity of the SVD is $\mathcal{O}(r^3)$.
So, the overall of computational complexity of Algorithm~\ref{alg:rmc} is dominated by
the linear system solving.
When the number of observations in certain row/column is much larger than $r$,
we may randomly choose $\mathcal{O}(r)$ observations from the row/column,
then solve a much smaller linear system of equations.
Again, the overall computational complexity in each iteration is $\mathcal{O}((m+n)r^3)$.

Also, note that the linear systems on Line 8 and 13 can be solved in parallel.
Therefore, Algorithm~\ref{alg:rmc} are suitable for large scale problems.\\

\begin{remark}
When $\Omega_t$ is fixed, Algorithm~\ref{alg:rmc} essentially minimizes $\|\PT{t}(XY^{\T}-M)\|_F$ via ADM.
If gradient method is used to minimize $\|\PT{t}(XY^{\T}-M)\|_F$, Algorithm~\ref{alg:rmc} is  similar to the GD method in~\citep{Proc:Yi_NIPS16},
except the regularization term $\|U_t^{\T}U_t-V_t^{\T}V_t\|_F$ in the loss function.
\end{remark}

\vspace{-0.07in}
\section{Convergence}\label{sec:convergence}
\vspace{-0.07in}

This section analyzes the convergence of Algorithm~\ref{alg:rmc}.
We first study the full observation case, which serves as a motivation for the partial observation case next.

To present the results, we need to define the $k$ {\em canonical angles}.
Let $\mathcal{X}, \mathcal{Y}$ be two $k$-dimensional subspaces of $\R^n$.
Let $X, Y \in\R^{n\times k}$ be the orthonormal basis matrices of
$\mathcal{X}$ and $\mathcal{Y}$, respectively, i.e.,
\[
\mathcal{R}(X) = \mathcal{X},\ X^{\T}X=I_k,
\ \mbox{ and } \
\mathcal{R}(Y) = \mathcal{Y},\ Y^{\T}Y=I_k.
\]
Denote $\omega_j$ for $1\le j\le k$ the singular values of $Y^{\T}X$
in  ascending order, i.e., $\omega_1\le\dots\le\omega_k$.
The $k$ {\em canonical angles $\theta_j(\mathcal{X},\mathcal{Y})$
    between $\mathcal{X}$ and $\mathcal{Y}$ } are defined by
\begin{equation}\label{eq:indv-angles-XY}
0\le\theta_j(\mathcal{X},\mathcal{Y}):=\arccos\omega_j\le\frac {\pi}2\quad\mbox{for $1\le j\le k$}.
\end{equation}
They are in  descending order, i.e., $\theta_1(\mathcal{X},\mathcal{Y})\ge\cdots\ge\theta_k(\mathcal{X},\mathcal{Y})$.
Set
\begin{equation}\label{eq:mat-angles-XY}
\Theta(\mathcal{X},\mathcal{Y})=\diag(\theta_1(\mathcal{X},\mathcal{Y}),\ldots,\theta_k(\mathcal{X},\mathcal{Y})).
\end{equation}
In what follows, we sometimes place a vector or matrix in one or both
arguments of $\theta_j(\,\cdot\,,\,\cdot\,)$ and $\Theta(\,\cdot\,,\,\cdot\,)$ with the meaning that
it is about the subspace spanned by the vector or the columns of the matrix argument.

\vspace{-0.05in}
\subsection{Full observation case}

\begin{theorem}\label{thm1}
Let $M=L_*+S_*\in\R^{m\times n}$ ($m\ge n$), where $L_*$ is low rank, i.e., $r=\rank(L_*)\ll n$. 
Let the SVD of $M$ be $M=U{\Sigma}{V}^{\T}$,
where ${U}=[{U}_1\,|\,{U}_2]=[{u}_1,\dots,{u}_r\,|\, {u}_{r+1},\dots,{u}_m]$,
 ${V}=[{V}_1\,|\,{V}_2]=[{v}_1,\dots,{v}_r\,|\, {v}_{r+1},\dots,{v}_n]$ are orthogonal matrices,
${\Sigma}=\begin{bmatrix}\diag({\Sigma}_1,{\Sigma}_2)\\ 0\end{bmatrix}$,
${\Sigma}_1=\diag({\sigma}_1,\dots,{\sigma}_r)$,
${\Sigma}_2=\diag({\sigma}_{r+1},\dots,{\sigma}_n)$,
and ${\sigma}_1\ge\dots\ge{\sigma}_n$.
Let $M_r={U}_1{\Sigma}_1{V}_1^{\T}$ be the best rank $r$ approximation of $M$.
Let the economy sized SVD of $L$ be $L_*=U_*\Sigma_*V_*^{\T}$,
where $U_*\in\R^{m\times r}$ and $V_*\in\R^{n\times r}$ both have orthonormal columns,
$\Sigma_*=\diag({\sigma}_{1*},\dots,{\sigma}_{r*})$ with ${\sigma}_{1*}\ge\dots\ge {\sigma}_{r*}>0$.
If
\begin{subequations}\label{sss}
\begin{align}
&\|(I-U_*U_*^{\T})S_*(I-V_*V_*^{\T})\|<{\sigma}_{r*},\label{sss1}\\
&\max\{\|S_*V_*\|, \|S_*^{\T}U_*\|\}< {\sigma_r-\sigma_{r+1}},\label{sss2}
\end{align}
\end{subequations}
then
\begin{align*}
&\max\{\theta_u,\theta_v\}\le \eta,\\
&\frac{\|L_*-M_r\|_{\max}}{\|L_*\|}\le  (\|U_*\|_{2,\infty}\theta_v + \|V_*\|_{2,\infty}\theta_u)\\
& \mbox{}\hskip.9in +  (1+3\|U_*\|_{2,\infty} \|V_*\|_{2,\infty})\theta_u\theta_v,
\end{align*}
where $\theta_u=\|\sin\Theta(U_1,U_*)\|$, $\theta_v=\|\sin\Theta(V_1,V_*)\|$ and $\eta=\frac{\max\{\|S_*V_*\|, \|S_*^{\T}U_*\|\}}{\sigma_r-\sigma_{r+1}-\max\{\|S_*V_*\|, \|S_*^{\T}U_*\|\}}$.
\end{theorem}

Theorem~\ref{thm1} tells that when \eqref{sss} holds and $\eta$ is small,
the principal angles between $U_1$ and $U_*$, $V_1$ and $V_*$ will be small,
and the best rank-$r$ approximation of $M$ is a good approximation of $L_*$.
Notice that, \eqref{sss} does not necessarily implies $\|S_*\|$ is small (compared with $\|L_*\|$).
In fact, we have the following example,
in which  $\|S_*\|$ is comparable with $\|L_*\|$ and $\eta=0$.\\

\begin{example}
Let $M=L_*+S_*$, $L_*=\frac{1}{n}\bf{1}_n \bf{1}_n^{\T}$, \\
$S_*=\frac{\rho}{4}\left[\begin{smallmatrix}
2 & -1 & 0&\dots & 0& -1\\
-1 & 2& -1 & 0 &\dots & 0\\
0 & -1 & 2 & -1 & \ddots & \vdots\\
\vdots & \ddots & \ddots & \ddots & \ddots &0 \\
0& \dots & 0&-1 & 2& -1\\
-1 & 0& \dots & 0 &-1 & 2\\
\end{smallmatrix}\right]$,
where ${\bf 1}_n$ is an $n$-by-$1$

vector of ones, $n$ is even, $\rho\in(-1,1)$ is a real parameter.
One can verify that
$\|L_*\|=1$, $\|S_*\|=\rho$, $\|S_*{\bf 1}_n\|=0$, the first two singular values of $M$ are $\sigma_1=1$ and $\sigma_2=|\rho|$,
and the economy sized SVD of $L$ can be given by $L=U_*\Sigma_*V_*^{\T}$,
where $U_*=V_*=\frac{1}{\sqrt{n}}{\bf 1}_n$, $\Sigma_*={\sigma}_{1*}=1$.
Then $\|(I-U_*U_*^{\T})S_*(I-V_*V_*^{\T})\|=\|S_*\|=|\rho|<1={\sigma}_{1*}$,
and $\max\{\|S_*V_*\|, \|S_*^{\T}U_*\|\}=0< 1-|\rho|=\sigma_1-\sigma_2$.
In other words, the assumption \eqref{sss} holds.
Noticing that $\eta=\frac{\max\{\|S_*V_*\|, \|S_*^{\T}U_*\|\}}{\sigma_r-\sigma_{r+1}-\max\{\|S_*V_*\|, \|S_*^{\T}U_*\|\}}=0$,
by Theorem~\ref{thm1}, we can conclude that $\|\sin\Theta(U_1,U_*)\|=\|\sin\Theta(V_1,V_*)\|=0$, $M_1=L_*$,
where $U_1$, $V_1$ are the top left and right singular vectors of $M$, respectively,
and $M_1$ is the best rank-1 approximation of $M$.
\end{example}

Now let us assume all entries of $M$ are observed,
 $S_*$ is sufficiently small, and can be taken as a perturbation to $L_*$.
Let $Y_t$ be guess for $Y$ and have orthonormal columns, let us perform one iteration of ADM. For the ease of illustration, also assume that $r_{t+1}=r_t$.
Then the iteration reads:
{\bf (S1)} $\wtd{X}_{t+1}=MY_t$;
{\bf (S2)} $\wht{X}_{t+1}=MY_t G_x$, where $G_x$ is such that $\wht{X}_{t+1}$ is orthonormal;
{\bf (S3)} $\wtd{Y}_{t+1}=M^{\T} \wht{X}^{t+1}$;
{\bf (S4)} $\wht{Y}_{t+1}=M^{\T} \wht{X}^{t+1} G_y$, where $G_y$ is such that $\wht{Y}_{t+1}$ is orthonormal.
So, we get
\begin{align}
\wht{Y}_{t+1}=(M^{\T}M)Y_tG_xG_y,
\end{align}
which is just one iteration of subspace iteration (a generation of power iteration)
for computing the dominant eigenspace of $M^{\T}M$ (e.g., \citep{demmel1997applied,stewart2001matrix,van2012matrix}).
In fact, {\bf (S2)} and {\bf (S4)} are iterations for subspaces spanned by the dominant left and right singular vectors of $M$, respectively.
Classic results tell that the subspaces $\mathcal{R}(X_t)$ and $\mathcal{R}(Y_t)$ converge to the subspaces spanned by the left and right  singular vectors of $M$ corresponding with the dominant singular values. When the perturbation is small,
$\mathcal{R}(X_t)$ and $\mathcal{R}(Y_t)$ are good approximations for
$\mathcal{R}(L_*)$ and $\mathcal{R}(L_*^{\T})$, respectively.
In particular, when $S_*=0$ and $V_*^{\T}Y_t$ is nonsingular, we have $\|\sin\Theta(X_{t+1},U_*)\|=0$, $\|\sin\Theta(Y_{t+1},V_*)\|=0$, i.e., one iteration of ADM gives the true solution.

\vspace{-0.05in}
\subsection{Partial observation case}
For the partial observation case, we study the convergence of Algorithm~\ref{alg:rmc}
under the following assumptions:
\begin{enumerate}\vspace{-0.1in}
\item[{\bf (A1)}]
For $L_*$, the column  and row incoherence conditions with parameter $\mu$ hold, i.e.,
\begin{align*}
\max_{1\le i\le m}\|U_*^{\T}e_i\|^2\le \frac{\mu r}{m},\qquad \max_{1\le j\le n}\|V_*^{\T}e_j\|^2 \le \frac{\mu r}{n}.
\end{align*}
\item [{\bf (A2)}] $S_*$ has at most $\varrho$-fraction nonzero entries per row and column, i.e.,
\begin{align*}
&\|S_{*(i,:)}\|_0\le \varrho n,\quad  \|S_{*(:,j)}\|_0\le \varrho m, \ \mbox{ for all } i,j.
\end{align*}
\item[{\bf (A3)}] Each entry of $M$ is observed independently with probability $p$.
\end{enumerate}
Besides the notations in Algorithm~\ref{alg:rmc},
we also adopt the following notations:
\begin{align*}
\theta_{x,t}=\|\sin\Theta(X_t,U_*)\|,\quad \theta_{y,t}=\|\sin\Theta(Y_t,V_*)\|.
\end{align*}

The following theorem tells that the SVD of the partial observed matrix (after the removal of outlier) indeed gives good approximation for $U_*$, $V_*$.
Furthermore, $X_1$, $Y_1$ satisfy a incoherence condition with parameter $\mu_1$,
$\|L_*-X_1\Sigma_1Y_1^{\T}\|_{\max}$ is bounded.

\vspace{0.1in}

\begin{theorem}\label{thm2}
Assume {\bf (A1)}, {\bf (A2)}, {\bf (A3)} and $m\ge n$.
Let $M=L_*+S_*\in\R^{m\times n}$ with $\rank(L_*)=r$,
$S_0$ be obtained as in Algorithm~\ref{alg:rmc}. Denote $r_{s}' = \frac{\|S_0-S_*\|_F^2}{\|S_0-S_*\|^2}$, $\gamma=\frac{2}{1-\varrho} \sqrt{\frac{2\varrho }{r_s' p}}$.
If
\begin{align}\label{cck}
(\xi+ \gamma )  \mu r \kappa<\frac16,
\end{align}
then with probability  $\ge 1-1/m^{10+\log \alpha}$, it holds that
\begin{align}\label{sinb}
\max\{\theta_{x,1},\theta_{y,1}\}\le 3 (\xi + \gamma)\mu r \kappa,
\end{align}
where $\xi=6\sqrt{\frac{\alpha }{p' n}}$,
$\kappa=\frac{{\sigma}_{1*}}{{\sigma}_{r*}}$. 
Further assume $\mu\ll n$ and that there exists a positive constant $\mu_1'\ll n$ such that
\begin{align}\label{cck2}
(\xi+\gamma)\mu r\kappa\le \frac{1}{3}\sqrt{\frac{\mu_1' r}{m}},
\end{align}
then
\begin{align*}
& \|X_1\|_{2,\infty}\le \sqrt{\frac{\mu_1 r}{m}}, \qquad \|Y_1\|_{2,\infty}\le \sqrt{\frac{\mu_1 r}{n}},\\
& \|L_*-X_1\Sigma_1Y_1^{\T}\|_{\max}\le \|L_*\| \Big( \sqrt{\frac{\mu r}{m}} \theta_{y,1} + \sqrt{\frac{\mu r}{n}} \theta_{x,1} \\
&\mbox{}\hskip1.65in+ \theta_{x,1}\theta_{y,1}\Big) +\mathcal{O}(n^{-3/2}),
\end{align*}
where $\mu_1=2(\mu +\mu_1')$.
\end{theorem}

\vspace{0.08in}

\begin{remark}
When there is no corruption, i.e., $\varrho=0$, then $\gamma=0$.
Furthermore, when $m=\mathcal{O}(n)\gg 1$, since $p'n=\mathcal{O}(1)r$, we know that $\xi=\sqrt{\frac{\alpha}{p'n}}$ is small, the larger $p'n$ is, the smaller $\xi$ is.
By Theorem~\ref{thm1}, $\theta_{x,1}$ and $\theta_{y,1}$ will be small.
In other words, the SVD $\frac{1}{p'}\PT{0}(M-S_0)=X_1\Sigma_1Y_1^{\T}$,
gives good approximation for both $U_*$ and $V_*$, by $X_1$ and $Y_1$, respectively.
\end{remark}

The following theorem, which is motivated by \cite[Lemma 55]{drineas2018lectures}, establish the bridge between the full observation case and the partial observation case. This gives an upper bound for the distance between the least square solutions between the full observation case and the partial observation~case.

\vspace{0.08in}

\begin{theorem}\label{lem:xx}
Let $m\ge n$, and denote
\begin{align*}
X_{\opt}&=\argmin_X\|X{Y}_t^{\T}- (M-S_t)\|, \\
\wtd{X}_{\opt}&=\argmin_X\|\PT{t}(X{Y}_t^{\T}-(M-S_t))\|.
\end{align*}
Assume that
$\Omega_t$ can be obtained by sampling each entry of $M$ with probability  $p'=p(1-\varrho)$,
$\|{Y}_t\|_{2,\infty}\le \sqrt{\frac{\mu' r}{n}}$ for some $\mu'>0$,
and
\begin{align}\label{infx}
\inf_{X\in\R^{m\times r}}\frac{\|\PT{t}(XY_t^{\T})\|}{\|X\|}\ge \sigma
\end{align}
 for some constant $\sigma>0$.
Then w.p. $\ge 0.99$, it holds
\begin{align*}
\|\wtd{X}_{\opt}-X_{\opt}\| &\le \Big(\frac{2}{3}\log(m+n)+5\Big) \frac{\sqrt{\mu' rp'}}{\sigma^2}\|R\|_{\max},
\end{align*}
where $R=(M-S_t)(I-Y_tY_t^{\T})$.
\end{theorem}

\vspace{0.1in}

\begin{remark}
The requirement \eqref{infx} is critical.
The parameter $\sigma$ reflects the condition number of the least square problem on line 8 of Algorithm~\ref{alg:rmc}.
What's more, the larger $p'$ is, the larger $\sigma$ is (in particular, if $p'=1$, $\sigma=1$), the smaller the distance between $\wtd{X}_{\opt}$ and $X_{\opt}$ is,
which agrees with our intuition.
$R$ is the residual for the full observation case, i.e., $R=X_{\opt}Y_t^{\T}-(M-S_t)$.
If the residual is small, the distance between $\wtd{X}_{\opt}$ and $X_{\opt}$ will be small, too.\\
\end{remark}

\begin{definition}
Define $\mu'\triangleq\max\{\mu_u,\mu_v\}$, where
\begin{align*}
\mu_u&\triangleq\sup_{U\in\R^{m\times r}}\{\frac{m}{r} \|U\|_{2,\infty}^2 \; | \; \|\sin\Theta(U_*,U)\|\le \theta_{x,1}\},\\
\mu_v&\triangleq\sup_{V\in\R^{n\times r}}\{\frac{n}{r} \|V\|_{2,\infty}^2 \;| \; \|\sin\Theta(U_*,V)\|\le \theta_{y,1}\}.
\end{align*}
\end{definition}
By definition of $\mu'$, we know that if $\theta_{x,t}\le \theta_{x,1}$,
$\theta_{y,t}\le \theta_{y,1}$ for all $t$, then $X_t$, $Y_t$ satisfy the incoherence condition with parameter $\mu'$.
Recall Theorem~\ref{thm1},  under the assumption of \eqref{cck2}, $\theta_{x,1}$ and $\theta_{y,1}$ are quite small (at the order of $\frac{1}{\sqrt{m}}$), then we can show that $\mu'\le \mu_1$,
which implies that $\mu'$ is not large.

The next theorem establishes the convergence rate for the ADM,
which is the key  in our proof of Theorem~\ref{thm:main}.

\begin{theorem}\label{lem:theta}
Assume that
$\Omega_t$ can be obtained by sampling each entry of $M$ with probability  $p'$,
$\|{Y}_t\|_{2,\infty}\le \sqrt{\frac{\mu' r}{n}}$ for some $\mu'>0$,
\eqref{infx} and
\begin{align}
\|L_*-X_t\Sigma_tY_t^{\T}\|_{\max}\le c \|L_*\|\; \theta_{y,t} \sqrt{\frac{\mu r}{m}},\label{lxyt}
\end{align}
for some constants $c>0$.
Denote
\begin{align*}
r_s &= \inf_t\frac{\|S_t-S_*\|_F^2}{\|S_t-S_*\|^2}, \qquad \zeta=\sqrt{\frac{2s\mu r}{m r_s}},\\
C_{\rm LS}&=\Big(\frac{2}{3}\log(m+n)+5\Big) \frac{\sqrt{\mu' rp'}}{\sigma^2},\\
C&=C_{\rm LS}(1+ 2c\sqrt{2p\varrho n} ) \sqrt{\mu r},\\
\epsilon&=c\kappa\zeta, \qquad \phi=\frac{8\epsilon(\kappa+\sqrt{2}\epsilon)+\sqrt{2}C\kappa/\sqrt{m}}{1-2\epsilon-C\kappa/\sqrt{m}}.
\end{align*}
Further assume $\theta_{y,t}\le\frac{1}{\sqrt{2}}$,
then w.p. $\ge 0.99$,
\[
\theta_{x,t+1}\le \phi \; \theta_{y,t}.
\]
\end{theorem}

\begin{remark}
The assumption \eqref{lxyt} is not a strong requirement as it looks.
By Theorem~\ref{thm1}, \eqref{lxyt} is natural for $t=1$.
For general $t>1$, it can be shown that there exists a constant $c>0$ such that \eqref{lxyt} holds (see supplementary for details), as long as $X_t$, $Y_t$ satisfy the incoherence condition.
In general, the constant $c$ is at the order of $\mathcal{O}(1)$.
\end{remark}

\vspace{0.05in}

\begin{remark}
The constant $r_s$ is the infimum of the stable rank of $S_t-S_*$.
If we take a random matrix, whose entries are i.i.d. drawn from a normal distribution
$\mathcal{N}(0,\sigma^2)$, to approximate $S_t-S_*$.
Numerically, we found that $r_s=\mathcal{O}(n)$.
Therefore, $\epsilon \approx c \kappa \sqrt{2\varrho p \mu r}$ is a small number as long as $\varrho$ is small, $\mu$ and $\kappa$ are not large.
When $p'$ is sufficiently large, $C=\mathcal{O}(1)$.
Consequently, when $m=\mathcal{O}(n)$ is large, $\phi\ll 1$,
which agrees with our conclusion in the full observation case.
\end{remark}

\vspace{0.05in}

\begin{remark}\label{rem:thetay}
Under similar assumptions as in Theorem~\ref{lem:theta}, it can also be shown that
$\theta_{y,t+1}\le \phi\, \theta_{x,t+1}$.
Then $\theta_{y,t+1}\le \phi^2 \theta_{y,t}$.
When $\phi<1$, $\{\theta_{y,t}\}_t$ is a monotonically decreasing sequence.
Combining it with the definition of $\mu'$, we know that $Y_t$ satisfies the incoherence condition with parameter $\mu'$.
Similarly, $X_t$ also satisfies the incoherence condition.
\end{remark}

\begin{theorem}\label{thm:main}
Assume {\bf (A1)}, {\bf (A2)}, {\bf (A3)} and $m\ge n$.
Assume that
$\Omega_t$ can be obtained by sampling each entry of $M$ with probability  $p'=p(1-\varrho)$, $r_t\equiv r$ and
\begin{align*}
\inf_{X\in\R^{m\times r}}\frac{\|\PT{t}(XY_t^{\T})\|}{\|X\|}\ge \sigma,\
\inf_{Y\in\R^{n\times r}}\frac{\|\PT{t}(\wht{X}_tY^{\T})\|}{\|Y\|}\ge \sigma
\end{align*}
for some $\sigma>0$.
Let $r_s$, $\zeta$, $C$, $\epsilon$ be the same as in Theorem~\ref{lem:theta}.
Then with high probability, it holds that
\begin{align*}
\|M-S_{t+1}-X_{t+1}\Sigma_{t+1}Y_{t+1}^{\T}\|\le \psi \|M-S_t-X_t\Sigma_tY_t^{\T}\|,
\end{align*}
where
\[
\psi= \frac{2\sqrt{2}(\kappa +  2 \epsilon \sqrt{\frac{\mu r}{m}}  +   \frac{C\kappa}{\sqrt{m}}) (8\epsilon(\kappa+\sqrt{2}\epsilon)+\sqrt{2}\frac{C\kappa}{\sqrt{m}}) }{(1 - 4\sqrt{2} \epsilon  \sqrt{\frac{\mu r}{m}})(1-2\epsilon-\frac{C\kappa}{\sqrt{m}})}.
\]
\end{theorem}

\begin{remark}
If $\psi<1$, then by Theorem~\ref{thm:main}, $\{\|M-S_t-X_t\Sigma_tY_t^{\T}\|\}_t$ is a monotonically deceasing sequence.
And in limit, with high probability, it holds
\[
\lim_{t\rightarrow\infty} \|M-S_t-X_t\Sigma_tY_t^{\T}\|=0.
\]
Recall the way we determine $S_t$, we get $(L_*-X_t\Sigma_tY_t^{\T})_{ij}=0$ for any $(i,j)\notin\supp(S_t)$. Then we can show that
\[
X_t\Sigma_tY_t^{\T}\rightarrow L_*, \quad S_t\rightarrow S_*,\quad \mbox{as}\quad t\rightarrow \infty,
\]
i.e., exact recovery is achieved, with high probability.
\end{remark}

\vspace{-0.07in}
\section{Experiments}\label{sec:numer}
\vspace{-0.07in}

We compare Algorithm~\ref{alg:rmc} (NLEQ) with the gradient descent (GD) method~\citep{Proc:Yi_NIPS16}
and the PG-RMC method~\citep{Proc:Cherapanamjeri_ICML17}. The codes of GD and PG-RMC are obtained from the lrslibrary~\citep{sobral2016lrslibrary} in Github.~\footnote{\url{https://github.com/andrewssobral/lrslibrary/tree/master/algorithms/mc}}

\vspace{-0.05in}
\subsection{Synthetic Data}
\vspace{-0.05in}

We generate the data matrix $M\in\R^{d\times d}$ as follows. The low rank matrix $L_*$ is given by $L_*=U_*V_*^{\T}$,
where the entries of $U_*, V_*\in\R^{d\times r}$ are drawn independently from the Gaussian distribution with mean zero, variance $1/d$.
Each entry of the sparse matrix $S_*$ are nonzero with probability $\rho$, and the nonzero entries of $S_*$ are uniformly drawn from $[-\frac{r}{2d},\frac{r}{2d}]$.
Each entry of $M=L_*+S_*$ is observed independently with probability $p$.

\begin{figure}[!h]
\vskip -0.1in
\begin{center}
\includegraphics[width=0.42\textwidth]{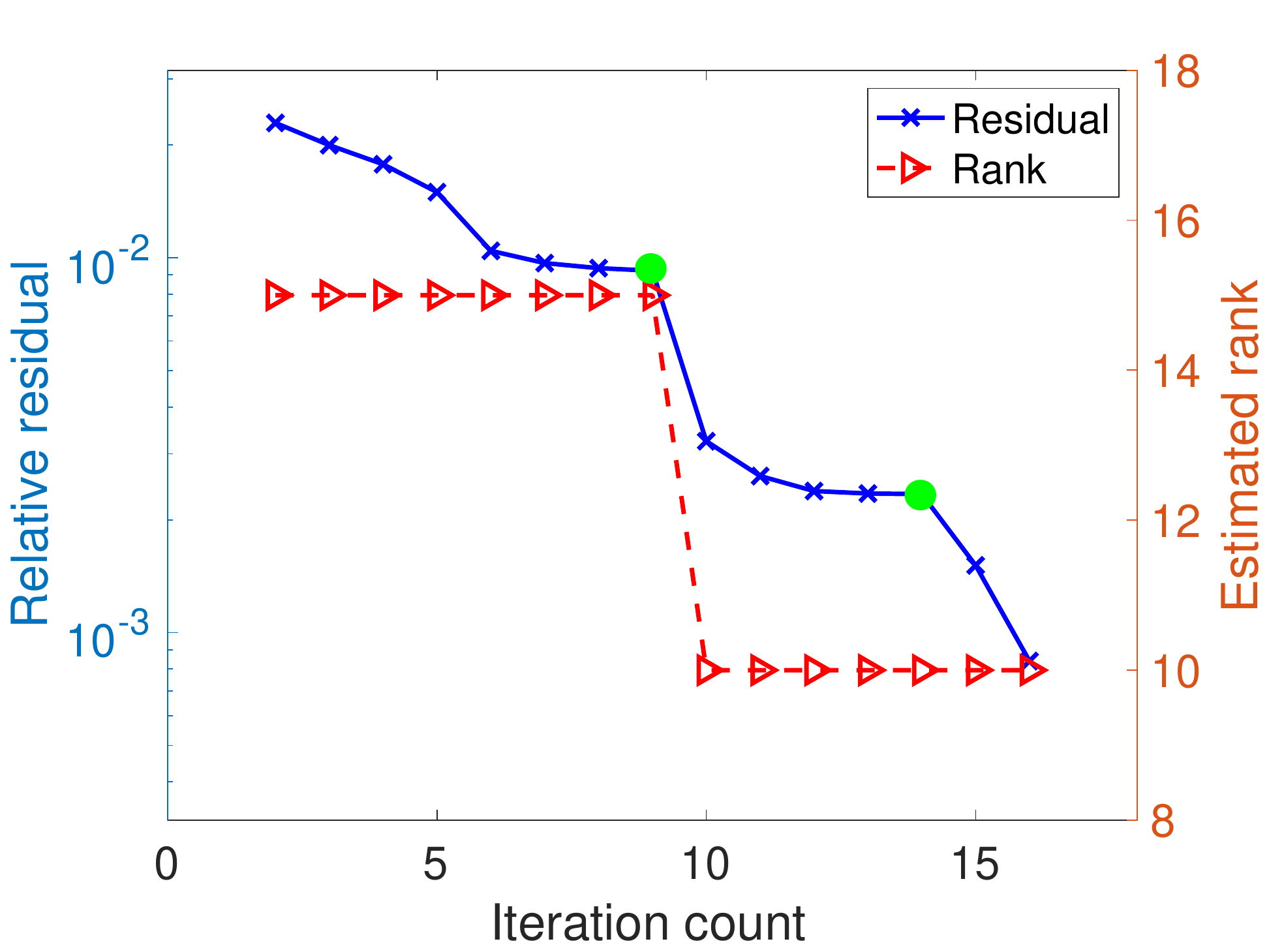}\\
\includegraphics[width=0.42\textwidth]{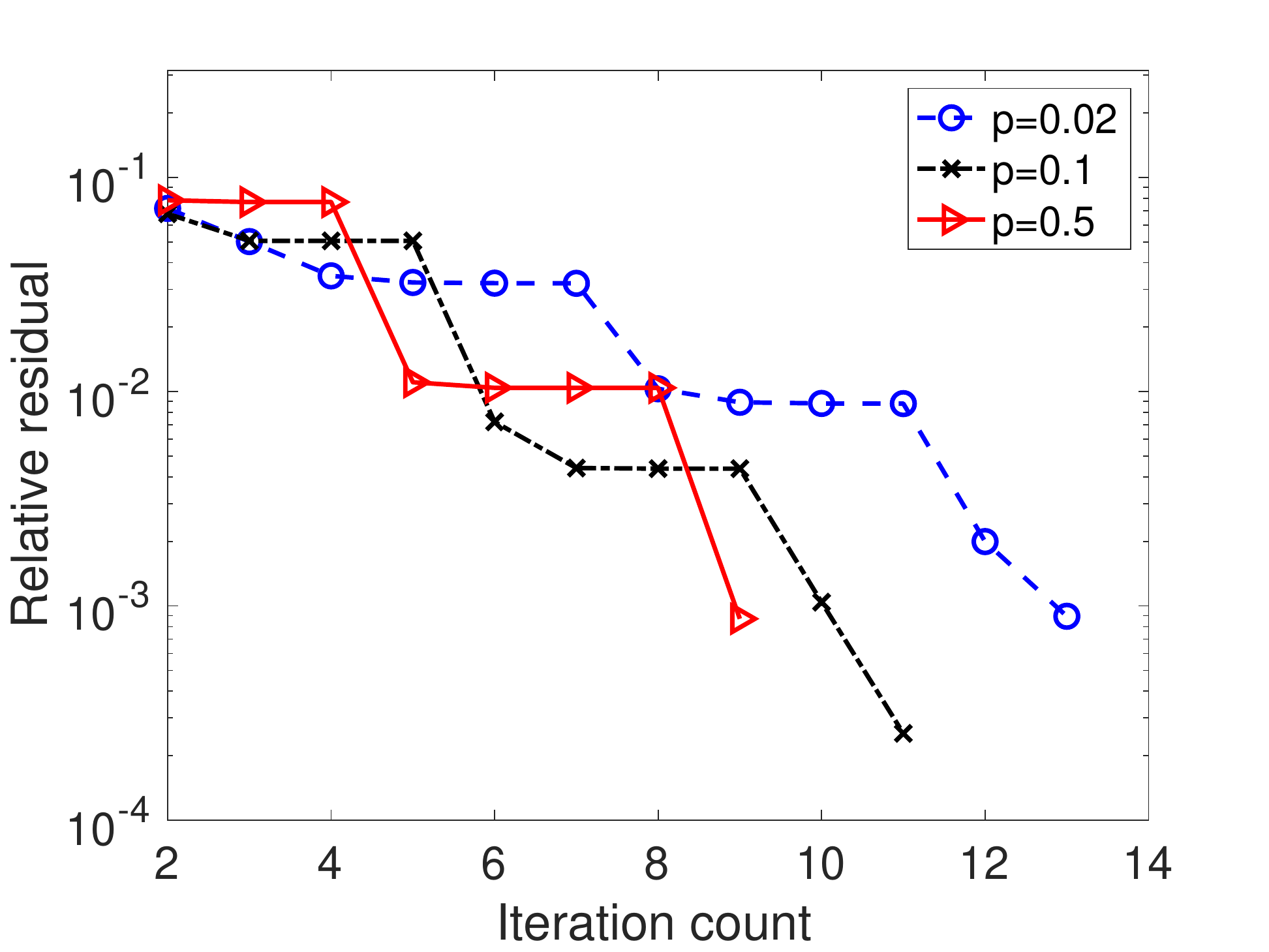}\\
\includegraphics[width=0.42\textwidth]{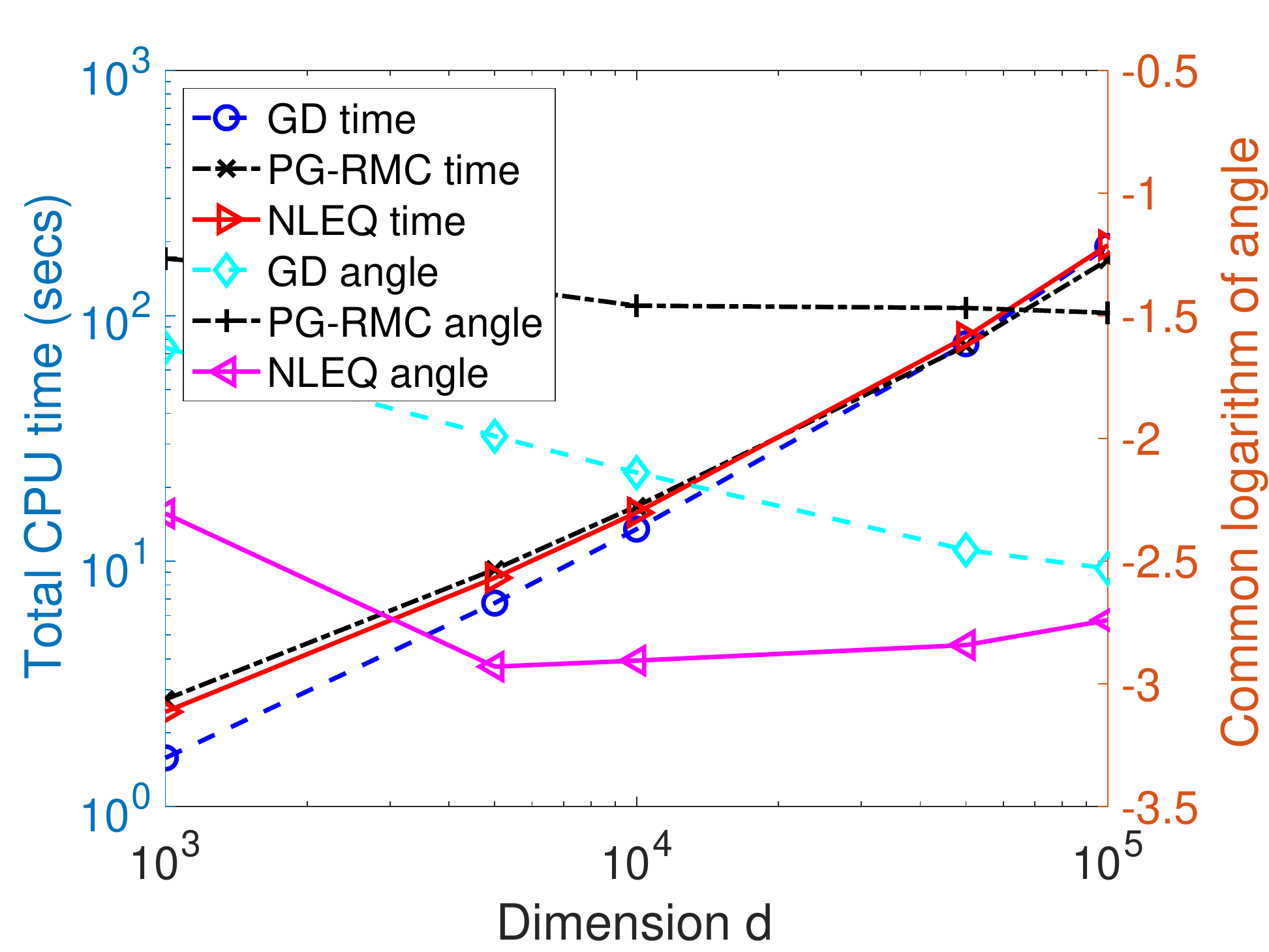}\\
\includegraphics[width=0.42\textwidth]{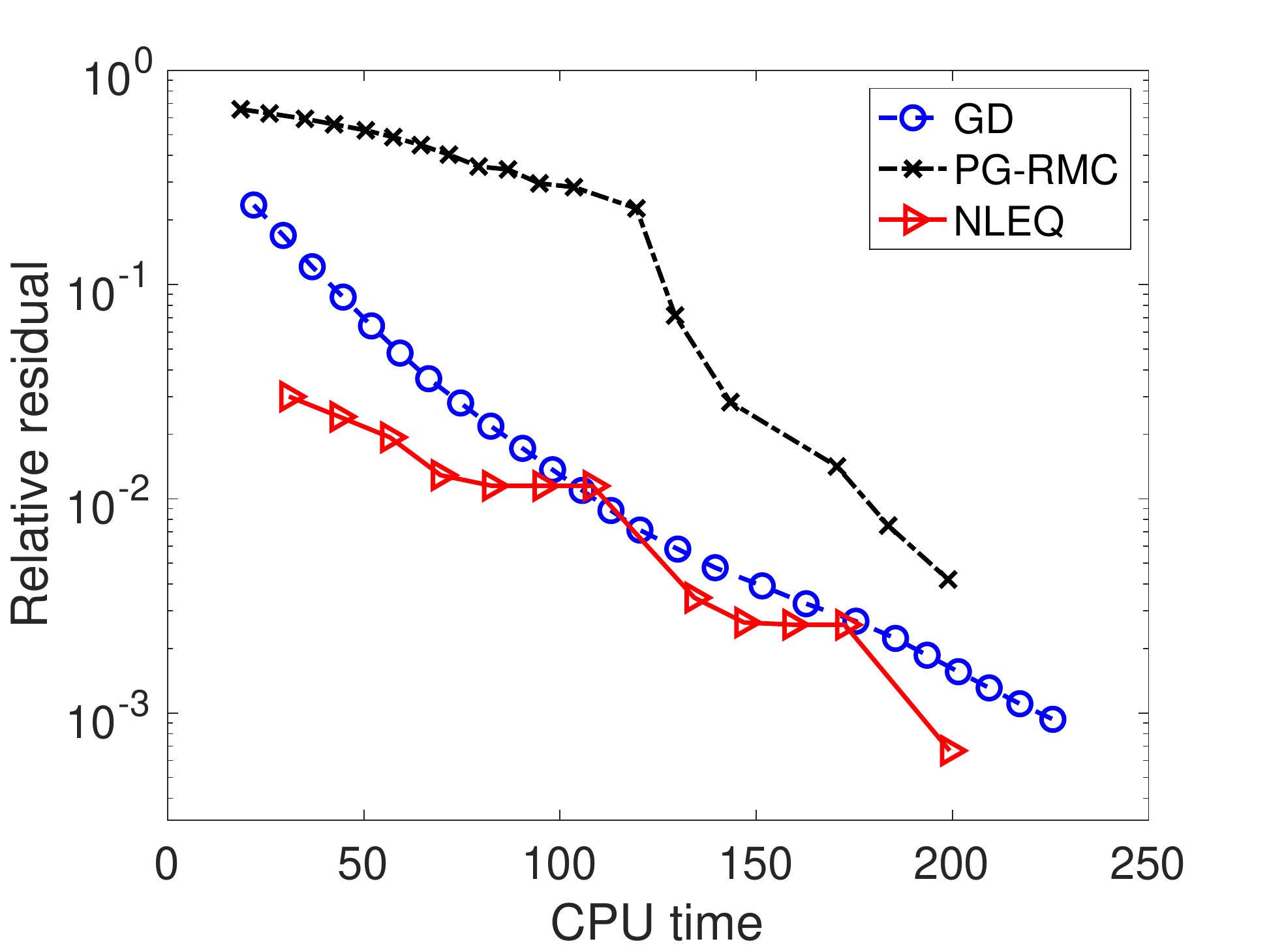}
\end{center}
\vspace{-0.2in}
\caption{Results on synthetic data. Up to down: (a)-(d).
(a) $d=1e5$, $r=10$, $p=0.0015$,  $\rho=0.1$.
(b) $d=1e4$, $r=10$, $\rho=0.1$, $p=0.02,0.1,0.5$.
(c) $d=1e3, 5e3,1e4,5e4,1e5$, $r=10$, $p=0.15r^2\log(m)/m$, $\rho=0.1$.
(d) $d=1e5$, $r=10$, $p=0.002$,  $\rho=0.1$.}
\vspace{-0.4in}
\label{fig:1}
\end{figure}

The results are presented in Figure~\ref{fig:1}.
In Figure~\ref{fig:1}-(a), we draw the relative residual $\frac{\|R_t\|_F}{\|\PT{t}(M)\|_F}$ and rank estimation $r_t$ vs. iteration number.
We can see that when the initial rank is larger than the true rank, as the iteration continues,
the true rank can be revealed from the singular values of $X_t(Y_t)^{\T}$, then the rank estimation drops to the true rank;
when the residual stagnates (represented by a big solid dot in the plot), outliers are removed and the residual decreases until convergence.
In Figure~\ref{fig:1}-(b), we plot the relative residual $\frac{\|R_t\|_F}{\|\PT{t}(M)\|_F}$ vs. total CPU time for different $p$.
We can see that Algorithm~\ref{alg:rmc} works for all three cases, the larger $p$ is, the less iteration number is needed.
In Figure~\ref{fig:1}-(c), we plot total CPU time and the angle $\max\{\|\sin\Theta(X_t,U_*)\|,\|\sin\Theta(Y_t,V_*)\|\}$ vs. matrix size $d$ for different methods.
We can see that the CPU time of all three methods grows linearly with respect to the matrix size, and are comparable with each other.
The angles of the three methods are all small, which confirm that all methods give the correct results; the angle produced by NLEQ is the smallest.
In Figure~\ref{fig:1}-(d), we plot the relative residual $\frac{\|R_t\|_F}{\|\POT(M)\|}$ vs. CPU time for all three methods.
The convergence behaviors of three methods are quite different: GD converges almost linearly;
PG-RMC at the beginning stage converges linearly with a low converge rate,
then converges almost linearly with a larger rate;
NLEQ has a zig-zag convergence, which is due to the removal of outliers.

\subsection{Foreground-background separation}

The next task is foreground-background separation.
By stacking up the vectorized video frames, we get a full data matrix. The static background will form a low rank matrix while the foreground can be taken as the sparse component. We apply our method NLEQ, and also GD and PG-RMC to two public benchmarks, the {\em Bootstrap} and {\em ShoppingMall}.\footnote{\url{
http://vis-www.cs.umass.edu/~narayana/castanza/I2Rdataset/}}
Each entry of the data matrix is observed independently w.p. $p=0.05$.
As presented in Figure~\ref{fig:2},
all three methods are able to separate the foreground from the background,
and the backgrounds obtained by three methods are similar. \\

\begin{figure}[!h]
\begin{turn}{90}
\quad\  Bootstrap
\end{turn}
\mbox{\hspace{-0.3in}
\includegraphics[height=0.22\textwidth,width=0.25\textwidth]{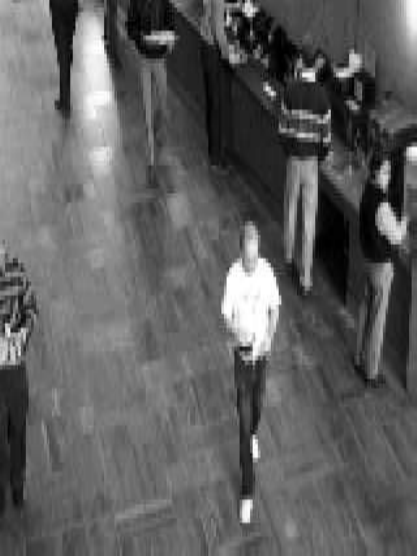}\
\includegraphics[height=0.22\textwidth,width=0.25\textwidth]{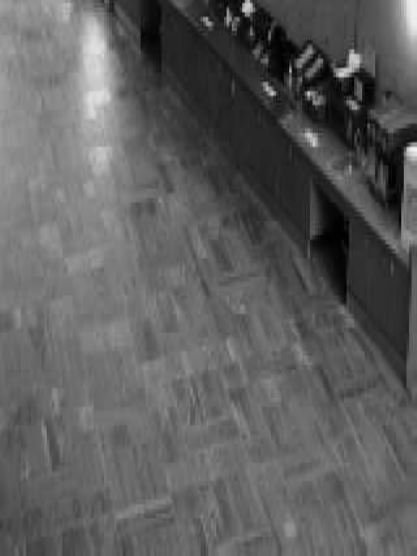}\\
\text{\hskip1.9cm Original \hskip 1.3cm  GD/PG-RMC/NLEQ}}\\
\begin{turn}{90}
\; ShoppingMall
\end{turn}
\mbox{\hspace{-0.3in}
\includegraphics[height=0.22\textwidth,
width=0.25\textwidth]{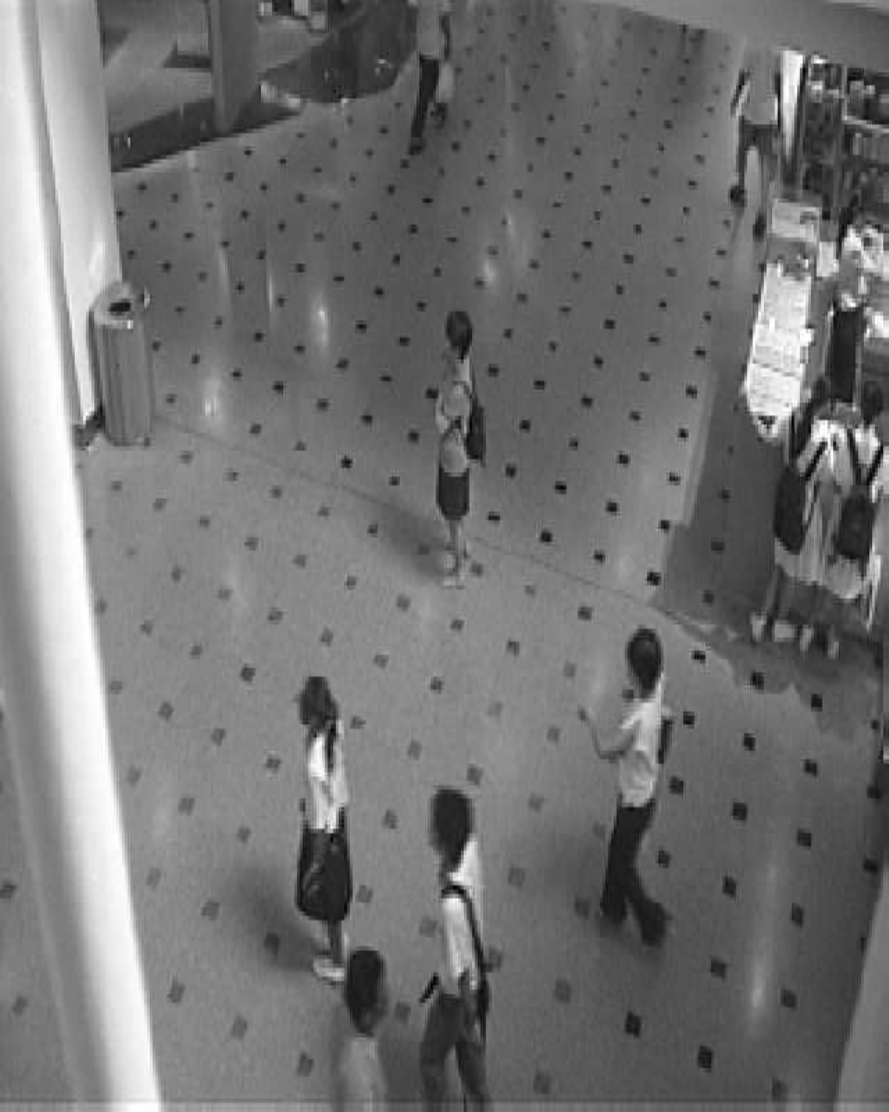}\
\includegraphics[height=0.22\textwidth,width=0.25\textwidth]{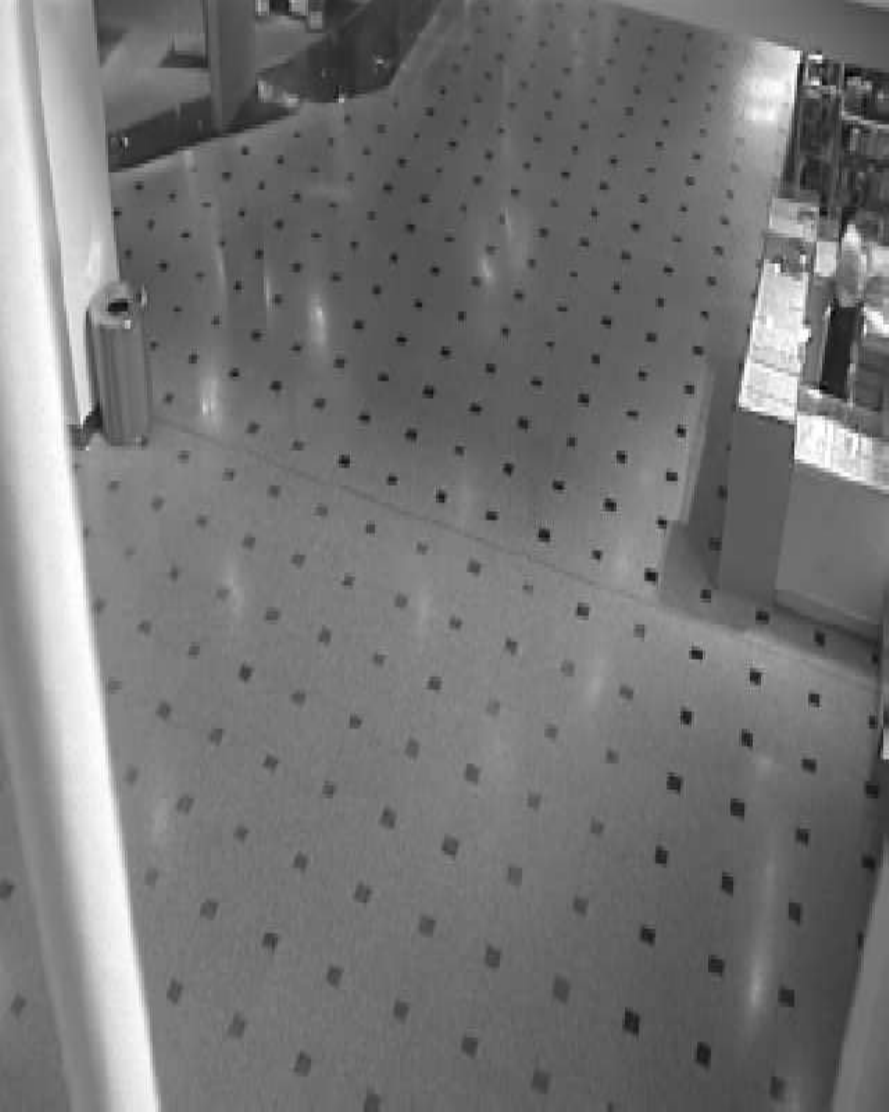}}\\
\text{\hskip1.9cm Original \hskip 1.3cm  GD/PG-RMC/NLEQ}\\
\vspace{-0.15in}
\caption{Foreground-background separation.}
\label{fig:2}
\end{figure}

\section{Conclusion}\label{sec:conclusion}

In this paper, we study the RMC problem from an algebraic point of view -- transform the RMC problem into a problem of solving an overdetermined nonlinear system of equations (with outliers).
This method does not require any objective function, convex relaxation or surrogate convex constraint.
Algorithmically, we propose to solve the NLEQ via ADM, in which the true rank and support set of the corruption are determined during the iteration.
The algorithm is highly parallelizable and suitable for large scale problems.
Theoretically, we characterize the sufficient conditions for when $L_*$ can be approximated by the low rank approximation of $M$  or $\frac1p\PO(M)$.
We establish sufficient conditions for $M_r\approx L_*$, where $M_r$ is the best rank $r$ approximation of the observed $M$.
The convergence of the algorithm is guaranteed, and exact recovery is achieved under proper assumptions.
Numerical simulations show that the algorithm is comparable with state-of-the-art methods in terms of efficiency and accuracy.

\clearpage
\balance
\bibliographystyle{plainnat}
\bibliography{myref,standard}

\onecolumn

{\centering\Large \bf Supplementary Materials}

\section{Preliminary lemmas}

The following lemma gives some fundamental results for $\sin\Theta(U,V)$,
which can be easily verified via definition.
\begin{lemma}\label{lem:sin}
    Let $[U,\, U_{\rm c}]$ and $[V,\, V_{\rm c}]$ be two orthogonal matrices with $U, V\in\R^{n\times k}$. Then
    \[
    \|\sin\Theta(U,V)\|_{\rm ui}=\|U_{\rm c}^{\T}V\|_{\rm ui}=\|U^{\T}V_{\rm c}\|_{\rm ui}. 
    \]
    Here $\|\cdot\|_{\rm ui}$ denotes any unitarily invariant norm, including the spectral norm and Frobenius norm.
    In particular, for the spectral norm, it holds $\|\sin\Theta(U,V)\|=\|UU^{\T}-VV^{\T}\|$;
    for the Frobenius norm, it holds $\|\sin\Theta(U,V)\|_F=\frac{1}{\sqrt{2}}\|UU^{\T}-VV^{\T}\|_F$.
\end{lemma}

The following lemma is the well-known Weyl theorem, which gives the perturbation bound for eigenvalues of Hermitian matrix.

\begin{lemma}\citep[p.203]{stewart1990matrix}
    \label{lem:eig}
    For two Hermitian matrices $A,\,\wtd{A}\in\mathbb{C}^{n\times n}$,
    let $\lambda_1\le \dots\le \lambda_n$, $\tilde{\lambda}_1\le \dots\le \tilde{\lambda}_n$ be eigenvalues of $A$, $\wtd{A}$, respectively.
    Then
    \[
    |\lambda_j-\tilde{\lambda}_j| \le \|A-\widetilde{A}\|,
    \quad\mbox{ for $1\le j\le n$}.
    \]
\end{lemma}

The following lemma is used to establish the perturbation bound for the invariant subspace of a Hermitian matrix,
which is due to Davis and Kahan.

\begin{lemma}\citep[Theorem 5.1]{davis1970rotation}\label{lem:syl}
    Let $H$ and $M$ be two Hermitian matrices, and let $S$ be a matrix of a compatible size as determined
    by the Sylvester equation
    $$
    HY-YM=S.
    $$
    If either all eigenvalues of $H$ are contained in a closed interval that
    contains no eigenvalue of $M$ or vice versa, then the Sylvester equation
    has a unique solution $Y$, and moreover
    $$
    \| Y\|_{\rm ui}\le\frac 1{\delta}\| S\|_{\rm ui},
    $$
    where $\delta=\min |\lambda-\omega|$ over all eigenvalues $\omega$ of $M$ and all eigenvalues $\lambda$ of $H$.

\end{lemma}

\medskip

For a rectangular matrix $A\in\R^{m\times n}$ (without loss of generality, assume $m\ge n$),
let the SVD of $A$ be $A=U\Sigma V^{\T}$,
where $U=[U_1 \,|\, U_2 \,|\, U_3]=[u_1,\dots,u_k\,|\, u_{k+1},\dots, u_r\,|\, u_{r+1},\dots,u_m]\in\R^{m\times m}$,
$V=[V_1\, |\, V_2\,|\, V_3]=[v_1,\dots, v_k\, |\, v_{k+1},\dots, v_r\,|\, v_{r+1},\dots,v_n]\in\R^{n\times n}$ are orthogonal matrices,
and $\Sigma=\begin{bmatrix}\diag(\Sigma_1,\Sigma_2) & 0_{r\times (n-r)} \\ 0_{(m-r)\times r} & 0_{(m-r)\times (n-r)} \end{bmatrix}$,
$\Sigma_1=\diag(\sigma_1,\dots,\sigma_k)$, $\Sigma_2=\diag(\sigma_{k+1},\dots,\sigma_r)$ with $\sigma_1\ge\dots\ge \sigma_r>0$,
$k\le r=\rank(A)$.
Then the spectral decomposition of $\begin{bmatrix}0 &A \\ A^{\T} & 0\end{bmatrix}$ can be given by
\begin{equation}\label{wj}
\begin{bmatrix}0 &A \\ A^{\T} & 0\end{bmatrix}=
X \diag(\Sigma_1,\Sigma_2,-\Sigma_1, -\Sigma_2,0_{n-r} , 0_{m-r})X^{\T},
\end{equation}
where $X=\frac{1}{\sqrt{2}}\begin{bmatrix}U & -U  & 0 & \sqrt{2}U_3\\ V & V & \sqrt{2}V_3 & 0\end{bmatrix}$ is an orthogonal matrix.

With the help of \eqref{wj} and Lemmas~\ref{lem:eig} and \ref{lem:syl},
we are able to prove Lemma~\ref{lem:errbound},
which established an error bound for singular vectors.

\begin{lemma}\label{lem:errbound}
    Given $A\in\R^{m\times n}$ ($m\ge n$), let the SVD of $A$ be given as above.
    Let $\hat{\sigma}_j$, $\hat{u}_j$, $\hat{v}_j$ be respectively the approximate singular values, right and left singular vectors of $A$
    satisfying that
    $\wht{U}=[\hat{u}_1,\dots,\hat{u}_k]\in\R^{m\times k}$ and $\wht{V}=[\hat{v}_1,\dots,\hat{v}_k]\in\R^{n\times k}$ are both orthonormal ,
    $\wht{\Sigma}=\wht{U}^{\T}A\wht{V}=\diag(\hat{\sigma}_1,\dots,\hat{\sigma}_k)$ with $\hat{\sigma}_1\ge \dots\ge \hat{\sigma}_k>0$.
    Let
    \begin{equation}\label{eq:EF}
    E=A\wht{V}-\wht{U}\wht{\Sigma}, \qquad F=A^{\T}\wht{U}-\wht{V}\wht{\Sigma}.
    \end{equation}
    If
    \begin{align*}
\|(I_m-\wht{U}\wht{U}^{\T})A(I_n-\wht{V}\wht{V}^{\T})\|<\hat{\sigma}_k,\qquad \max\{\|E\|,\|F\|\}<\sigma_k-\sigma_{k+1},
    \end{align*}
    then
    \begin{align*}
    \max\{\Theta_u, \Theta_v\}\le \eta, 
    \qquad \frac{\|U_1\Sigma_1V_1^{\T} - \wht{U}\wht{\Sigma}\wht{V}^{\T}\|_{\max}}{\|A\|}
   \le (\|U_1\|_{2,\infty}\Theta_v + \|V_1\|_{2,\infty}\Theta_u) +  (1+3\|U_1\|_{2,\infty} \|V_1\|_{2,\infty})\Theta_u\Theta_v,
    \end{align*}
where $\Theta_u=\|\sin\Theta(U_1,\wht{U})\|$, $\Theta_v=\|\sin\Theta(V_1,\wht{V})\|$,
$\eta=\frac{\max\{\|E\|,\|F\|\}}{\sigma_k-\sigma_{k+1}-\max\{\|E\|,\|F\|\}}$.
\end{lemma}

\begin{proof}
    Let
    \begin{align*}
    &H=\begin{bmatrix} 0 & A\\ A^{\T} & 0\end{bmatrix},
    &&{X}_1=\frac{1}{\sqrt{2}}\begin{bmatrix}U_1 & -U_1\\ V_1 & V_1\end{bmatrix},\\
    &\wht{X}_1=\frac{1}{\sqrt{2}}\begin{bmatrix}\wht{U} & -\wht{U}\\ \wht{V} & \wht{V}\end{bmatrix},
    &&X_2=\frac{1}{\sqrt{2}}\begin{bmatrix}U_2 & -U_2  & 0 & \sqrt{2}U_3\\ V_2 & V_2 & \sqrt{2}V_3 & 0\end{bmatrix}.
    \end{align*}
 By calculations, we have
    \begin{align}\label{sinx1x1}
  \|X_1X_1^{\T} - \wht{X}_1\wht{X}_1^{\T}\|_{\rm ui}
    =\|\diag(U_1U_1^{\T}-\wht{U}\wht{U}^{\T}, V_1V_1^{\T}-\wht{V}\wht{V}^{\T})\|_{\rm ui}
    \end{align}

    By simple calculations, we have
    \begin{subequations}
        \begin{align}
        &H\wht{X}_1-\wht{X}_1\diag(\wht{\Sigma},-\wht{\Sigma})
        =\frac{1}{\sqrt{2}}\begin{bmatrix}A\wht{V}-\wht{U}\wht{\Sigma} & A\wht{V}-\wht{U}\wht{\Sigma}\\ A^{\T}\wht{U}-\wht{V}\wht{\Sigma}& -A^{\T}\wht{U}+\wht{V}\wht{\Sigma}\end{bmatrix}
        =\frac{1}{\sqrt{2}}\begin{bmatrix}E & E\\ F& -F\end{bmatrix}\triangleq R,\label{eq:r}\\
        &H X_2 - X_2\diag(\Sigma_2,-\Sigma_2,0,0)
        =0,\label{hx2}
        \end{align}
    \end{subequations}
    where \eqref{eq:r} uses \eqref{eq:EF},  \eqref{hx2} uses the SVD of $A$.
    Then it follows from \eqref{eq:r} that
    \begin{align}
    \|R\|
    &=\left\|\diag(E,F) \frac{1}{\sqrt{2}}\begin{bmatrix}I_k & I_k \\ I_k & -I_k\end{bmatrix}\right\|
    =\|\diag(E,F)\|=\max\{\|E\|,\|F\|\}.\label{normr}
    \end{align}

    Pre-multiplying \eqref{eq:r} by $X_2^{\T}$ and using \eqref{hx2}, we have
    \begin{align}\label{x2r}
    X_2^{\T}R&=X_2^{\T}H\wht{X}_1-X_2^{\T}\wht{X}_1\diag(\wht{\Sigma},-\wht{\Sigma})
    =\diag(\Sigma_2,-\Sigma_2,0,0)X_2^{\T}\wht{X}_1 - X_2^{\T}\wht{X}_1\diag(\wht{\Sigma},-\wht{\Sigma}).
    \end{align}
    To apply Lemma~\ref{lem:syl} to \eqref{x2r}, we need to estimate the gap between the eigenvalues of $\diag(\wht{\Sigma},-\wht{\Sigma})$
    and those of $\diag(\Sigma_2,-\Sigma_2,0,0)$.
    Using \eqref{eq:r} and $\wht{U}^{\T}A\wht{V}=\wht{\Sigma}$, we have
    \begin{align}
    (H-R\wht{X}_1^{\T}-\wht{X}_1R^{\T})\wht{X}_1=H\wht{X}_1 - R=\wht{X}_1\diag(\wht{\Sigma},-\wht{\Sigma}),
    \end{align}
    which implies that $\pm\hat{\sigma}_j$ are eigenvalues of $H-R\wht{X}_1^{\T}-\wht{X}_1R^{\T}$,
    and the corresponding eigenvectors are $\frac{1}{\sqrt{2}}\left[\begin{smallmatrix}\pm\hat{u}_j\\ \hat{v}_j\end{smallmatrix}\right]$, for $j=1,\dots,k$.
    Next, we declare that $\hat{\sigma}_1,\dots,\hat{\sigma}_k$ are the $k$ largest eigenvalues of $H-R\wht{X}_1^{\T}-\wht{X}_1R^{\T}$.
    This is because
    \begin{align*}
    &\max_{\wht{X}_1^{\T}x=0}\frac{x^{\T}(H-R\wht{X}_1^{\T}-\wht{X}_1R^{\T})x}{x^{\T}x}\\
    \le& \|(I-\wht{X}_1\wht{X}_1^{\T})(H-R\wht{X}_1^{\T}-\wht{X}_1R^{\T})(I-\wht{X}_1\wht{X}_1^{\T})\|\\
    =& \|(I-\wht{X}_1\wht{X}_1^{\T})H(I-\wht{X}_1\wht{X}_1^{\T})\|\\
    =&\left\| \begin{bmatrix}I_m-\wht{U}\wht{U}^{\T} & 0\\ 0 & I_n-\wht{V}\wht{V}^{\T}\end{bmatrix} \begin{bmatrix}0 & A\\ A^{\T} & 0\end{bmatrix}\begin{bmatrix}I_m-\wht{U}\wht{U}^{\T} & 0\\ 0 & I_n-\wht{V}\wht{V}^{\T}\end{bmatrix}\right\|\\
    =& \|(I_m-\wht{U}\wht{U}^{\T})A(I_n-\wht{V}\wht{V}^{\T})\|<\hat{\sigma}_k.
    \end{align*}
    Therefore, by Lemma~\ref{lem:eig}, we have
    \begin{align}
    |\sigma_j-\hat{\sigma}_j|\le \|R\wht{X}_1^{\T}+\wht{X}_1R^{\T}\|,\quad \mbox{for } j=1,\dots,k.
    \end{align}
    Together with \eqref{normr}, we get
    \begin{align}
    |\sigma_j-\hat{\sigma}_j|
    &\le \|R\wht{X}_1^{\T}+\wht{X}_1R^{\T}\|
    =\max_j|\lambda_j([R, \wht{X}_1]\left[\begin{smallmatrix}\wht{X}_1^{\T}\\ R^{\T}\end{smallmatrix}\right])|
    =\max_j|\lambda_j(\left[\begin{smallmatrix}\wht{X}_1^{\T}\\ R^{\T}\end{smallmatrix}\right][R, \wht{X}_1])|\notag\\
    &=\max_j|\lambda_j(\left[\begin{smallmatrix}0 & I_k\\ R^{\T}R & 0\end{smallmatrix}\right])|
    =\|R\|=\max\{\|E\|,\|F\|\}.
    \end{align}
    Here we uses the property that for any two matrix $A\in \R^{m\times n}$, $B\in\R^{n\times m}$, the nonzero eigenvalues of $AB$ and $BA$ are the same.

    Now by the assumption that $\max\{\|E\|,\|F\|\}<\sigma_k-\sigma_{k+1}$, we have
    \begin{align}\label{gapss}
    \hat{\sigma}_k-\sigma_{k+1}=\sigma_k-\sigma_{k+1}+\hat{\sigma}_k-\sigma_k\ge \sigma_k-\sigma_{k+1} - \max\{\|E\|,\|F\|\}>0,
    \end{align}
    therefore, the eigenvalues of $\diag(\Sigma_2,-\Sigma_2,0,0)$ lie in $[-\sigma_{k+1},\sigma_{k+1}]$,
    which has no eigenvalues of  $\diag(\wht{\Sigma},-\wht{\Sigma})$.
    We are able to apply Lemma~\ref{lem:syl} to \eqref{x2r},
    which yields
    \begin{align}\label{x2x1}
    \|X_2^{\T}\wht{X}_1\|_{\rm ui}\le \frac{\|X_2^{\T}R\|_{\rm ui}}{\sigma_k-\sigma_{k+1} - \max\{\|E\|,\|F\|\}}.
    \end{align}

    Using \eqref{sinx1x1}, Lemma~\ref{lem:sin}, \eqref{gapss} and \eqref{x2x1}, we get
    \begin{align}
    &\max\{\Theta_u,\Theta_v\}
    =\|\sin\Theta(X_1,\wht{X}_1)\|
    =\|X_2^{\T}\wht{X}_1\|
    \le  \frac{\|X_2^{\T}R\|}{\sigma_k-\sigma_{k+1} - \max\{\|E\|,\|F\|\}}
    \le \eta.\label{eq:theta}
    \end{align}

    Let
    \begin{align}\label{eq:uvhat}
    \wht{U}=U\Gamma_u=[U_1, U_2, U_3]\left[\begin{smallmatrix}\Gamma_{u1}\\ \Gamma_{u2}\\ \Gamma_{u3}\end{smallmatrix}\right],
    \quad \wht{V}=V\Gamma_v=[V_1, V_2, V_3]\left[\begin{smallmatrix}\Gamma_{v1}\\ \Gamma_{v2}\\ \Gamma_{v3}\end{smallmatrix}\right],
    \end{align}
    where $\Gamma_{u1}\in\R^{k\times k}$, $\Gamma_{u2}\in\R^{(r-k)\times k}$, $\Gamma_{u3}\in\R^{(m-r)\times k}$,
    $\Gamma_{v1}\in\R^{k\times k}$, $\Gamma_{v2}\in\R^{(r-k) \times k}$, $\Gamma_{v3}\in\R^{(n-r) \times k}$,
    and
    $\left[\begin{smallmatrix}\Gamma_{u1}\\ \Gamma_{u2}\\ \Gamma_{u3}\end{smallmatrix}\right]$,
    $\left[\begin{smallmatrix}\Gamma_{v1}\\ \Gamma_{v3}\\ \Gamma_{v3}\end{smallmatrix}\right]$ are both orthonormal.
    By \eqref{eq:theta}, we have
    \begin{align}\label{gammabound}
    \left\|\left[\begin{smallmatrix} \Gamma_{u2}\\ \Gamma_{u3}\end{smallmatrix}\right]\right\|=\Theta_u,\quad
    \sigma_{\min}(\Gamma_{u1})=\sqrt{1-\Theta_u^2},\quad
    \left\|\left[\begin{smallmatrix} \Gamma_{u2}\\ \Gamma_{u3}\end{smallmatrix}\right]\right\|=\Theta_v,\quad
    \sigma_{\min}(\Gamma_{v1})= \sqrt{1-\Theta_v^2}.
    \end{align}
    Substituting \eqref{eq:uvhat} into $\wht{U}^{\T}A\wht{V}=\wht{\Sigma}$ and using the SVD of $A$, we have
    \begin{align}\label{gammahat}
    \wht{\Sigma}=[\Gamma_{u1}^{\T}, \Gamma_{u2}^{\T}, \Gamma_{u3}^{\T}]
    \diag(\Sigma_1,\Sigma_2,0_{(m-r)\times (n-r)})
    \left[\begin{smallmatrix}\Gamma_{v1}\\ \Gamma_{v2}\\ \Gamma_{v3}\end{smallmatrix}\right]
    =\Gamma_{u1}^{\T} \Sigma_1\Gamma_{v1} + \Gamma_{u2}^{\T} \Sigma_2\Gamma_{v2}.
    \end{align}

    Then it follows that
        \begin{align}
    \|\Sigma_1 - \Gamma_{u1}\wht{\Sigma}\Gamma_{v1}^{\T}\|
    =& \|(\Sigma_1 - \Gamma_{u1}\Gamma_{u1}^{\T} \Sigma_1)+ (\Gamma_{u1}\Gamma_{u1}^{\T} \Sigma_1-\Gamma_{u1}\Gamma_{u1}^{\T} \Sigma_1\Gamma_{v1}\Gamma_{v1}^{\T})
    - \Gamma_{u1}\Gamma_{u2}^{\T} \Sigma_2\Gamma_{v2}\Gamma_{v1}^{\T}\|\notag\\
    \le & \|I - \Gamma_{u1}\Gamma_{u1}^{\T}\| \|\Sigma_1\|
    +\|\Gamma_{u1}\Gamma_{u1}^{\T}\| \|I - \Gamma_{v1}\Gamma_{v1}^{\T}\| \|\Sigma_1\|
    +  \|\Gamma_{u2}\| \|\Gamma_{v2}\| \|\Sigma_2\|\notag\\
    \le&  (\Theta_u^2+\Theta_v^2+\Theta_u\Theta_v)\|\Sigma_1\|.\label{ssdif}
    \end{align}

    Finally, using \eqref{gammabound}, \eqref{gammahat}, \eqref{ssdif} and $\|\Gamma_{u1}\|\le 1$, $\|\Gamma_{v1}\|\le 1$,  $\|\wht{\Sigma}\|\le \|A\|$,
we have
\begin{align*}
\|U_1\Sigma_1V_1^{\T} - \wht{U}\wht{\Sigma}\wht{V}^{\T}\|_{\max}
&=\max_{i,j}|e_i^{\T}(U_1\Sigma_1V_1^{\T} - \wht{U}\wht{\Sigma}\wht{V}^{\T})e_j|\\
&=\max_{i,j}|e_i^{\T}(U_1\Sigma_1V_1^{\T} - U\Gamma_u \wht{\Sigma}\Gamma_{v}^{\T}{V}^{\T})e_j|\\
&\le \max_{i,j} |e_i^{\T}(U_1\Sigma_1V_1^{\T}-U_1\Gamma_{u1}\wht{\Sigma}\Gamma_{v1}^{\T}V_1^{\T})e_j|
+\|[U_2,U_3]\left[\begin{smallmatrix} \Gamma_{u2}\\ \Gamma_{u3}\end{smallmatrix}\right]
\wht{\Sigma}\left[\begin{smallmatrix} \Gamma_{v2}\\ \Gamma_{v3}\end{smallmatrix}\right]^{\T}[V_2,V_3]^{\T}\|\\
& + \max_{i,j} \left(|e_i^{\T} [U_2,U_3]\left[\begin{smallmatrix} \Gamma_{u2}\\ \Gamma_{u3}\end{smallmatrix}\right]
\wht{\Sigma} \Gamma_{v1}^{\T}V_1^{\T} e_j|
+|e_i^{\T}U_1 \Gamma_{u1} \wht{\Sigma}\left[\begin{smallmatrix} \Gamma_{v2}\\ \Gamma_{v3}\end{smallmatrix}\right]^{\T}[V_2,V_3]^{\T}e_j|\right)\\
&\le  \max_{i,j}\left(3 \|e_i^{\T}U_1\| \|e_j^{\T}V_1\|  \|A\| \Theta_u\Theta_v +   \|A\| \Theta_u\Theta_v +  \|e_j^{\T}V_1\|  \|A\| \Theta_u+ \|e_i^{\T}U_1\|   \|A\| \Theta_v \right)\\
&\le \|A\| \left((\|U_1\|_{2,\infty}\Theta_v + \|V_1\|_{2,\infty}\Theta_u) +  (1+3\|U_1\|_{2,\infty} \|V_1\|_{2,\infty})\Theta_u\Theta_v\right),
\end{align*}
completing the proof.
\end{proof}

 \begin{lemma}\label{lem:bernstein}\citep[Corollary 6.1.2]{tropp2015introduction}
 Let ${\bf S}_1,\dots, {\bf S}_n$ be independent random matrices with common dimension $d_1\times d_2$,
 and assume that each matrix has uniformly bounded deviation from its mean:
 \[
\|{\bf S_k} - \mathbb{E}({\bf S}_k)\|\le L, \quad \mbox{ for each } k=1,\dots, n.
 \]
 Let ${\bf Z}=\sum_{k=1}^n {\bf S}_k$, $v({\bf Z})$ denote the matrix covariance statistic of the sum:
 \begin{align*}
 v({\bf Z})
 &=\max\{\|\mathbb{E}[({\bf Z} -\mathbb{E}(Z))({\bf Z} -\mathbb{E}(Z))^{\H}]\|, \|\mathbb{E}[({\bf Z} -\mathbb{E}(Z))^{\H}({\bf Z} -\mathbb{E}(Z))]\|\}\\
 &=\max\{\|\mathbb{E}[\sum_{k=1}^n({\bf S}_k-\mathbb{E}({\bf S}_k))({\bf S}_k-\mathbb{E}({\bf S}_k))^{\H}]\|,
\|\mathbb{E}[\sum_{k=1}^n({\bf S}_k-\mathbb{E}({\bf S}_k))^{\H}({\bf S}_k-\mathbb{E}({\bf S}_k))]\|\}.
 \end{align*}
 Then for all $t\ge 0$,
 \[
 \mathbb{P}\{\|{\bf Z}-\mathbb{E}({\bf Z})\|\ge t\}\le (d_1+d_2)\cdot \exp\Big(\frac{-t^2/2}{v({\bf Z})+Lt/3}\Big).
 \]
 \end{lemma}

 \begin{lemma}\label{lem:2fro}
For any linear homogeneous function $F:\R^{k}\rightarrow \R^{m\times n}$,
assume that the linear system of equations $F(x)=C$ either has a unique solution or has no solution at all.
Then it holds
\[
\argmin_x \|F(x)-C\|=\argmin_x\|F(x)-C\|_F.
\]
\end{lemma}

\begin{proof}
For any $A,B\in\R^{m\times n}$,
define $\langle A,B \rangle = \tr(A^{\T}B)$.
It is easy to see that $\langle\cdot,\cdot \rangle$ is an inner product over $\R^{m\times n}$.
Denote the range space of $F(\cdot)$ by $\mathcal{F}$,
and its orthogonal complement space by $\mathcal{F}^{\bot}$.
Write $C=C_{\rm LS}+C$ such that $C_{\rm LS}\in\mathcal{F}$, and $C\in\mathcal{F}^{\bot}$.
Then the solutions to $\min\|F(x)-C\|$ and $\min\|F(x)-C\|_F$ are nothing but the solutions to $F(x)=C_{\rm LS}$.
Since $C_{\rm LS}\in\mathcal{F}$, $F(x)=C_{\rm LS}$ has at least a solution.
By the assumption, the solution should be unique.
The proof is completed.
\end{proof}

\begin{lemma}\label{lem:lb}
Let $L_*\in\R^{m\times n}$ with $m\ge n$, let the SVD of $L_*$ be $L_*=U_*\Sigma_*V_*^{\T}$,
where $U_*\in\R^{m\times r}$, $V_*\in\R^{n\times r}$ are orthonormal, $\Sigma_*=\diag(\sigma_{1*},\dots,\sigma_{r*})$
with $\sigma_{1*}\ge\dots\ge\sigma_{r*}>0$.
Let $G\in\R^{m\times n}$ be a perturbation to $L_*$, $X\in\R^{m\times r}$, $Y\in\R^{n\times r}$ have full column rank.
Denote $\theta_x=\|\sin\Theta(U_*,X)\|$, $\theta_y=\|\sin\Theta(V_1,Y)\|$.
Then
\begin{align*}
\min_{X,Y}\|L_*-G-XY^{\T}\|
\ge  \sigma_{r*} \max\{\sqrt{1-\theta_x^2}\theta_y, \sqrt{1-\theta_y^2}\theta_x\} \sqrt{1-\theta_x^2}\sqrt{1-\theta_y^2} - \|G\|.
\end{align*}
\end{lemma}

\begin{proof}
Let $U_{*,c}$, $V_{*,c}$ be such that $U=[U_*,U_{*,c}]$, $V=[V_*,V_{*,c}]$ are orthogonal.
Let $\wht{X}=U_*C_x+U_{*,c}S_x$, $\wht{Y}=V_*C_y+V_{*,c}S_y$, where the columns of $\wht{X}$, $\wht{Y}$
form the orthonormal basis for $\mathcal{R}(X)$ and $\mathcal{R}(Y)$, respectively,
$C_x^{\T}C_x+S_x^{\T}S_x=I_r$, $C_y^{\T}C_y+S_y^{\T}S_y=I_r$.
By Lemma~\ref{lem:sin}, we know that $\|S_x\|=\theta_x$, $\|S_y\|=\theta_y$.

Noticing that
\begin{align*}
\min_{X,Y}\|L_*-XY^{\T}\|^2
&=\min_D\|U^{\T}L_*V-U^{\T}\wht{X} D\wht{Y}^{\T}V\|^2
=\min_D\left\| \begin{bmatrix} \Sigma_1 & 0  \\ 0 & 0 \end{bmatrix}
-\begin{bmatrix}C_x\\ S_x\end{bmatrix}D [C_y^{\T}, S_y^{\T}]\right\|^2\\
&=\left\| \begin{bmatrix} \Sigma_1 & 0  \\ 0 & 0 \end{bmatrix}
-\begin{bmatrix}C_x\\ S_x\end{bmatrix} [C_x, S_x]^{\T}  \begin{bmatrix} \Sigma_1 & 0  \\ 0 & 0 \end{bmatrix}
\begin{bmatrix}C_y\\ S_y\end{bmatrix}[C_y^{\T}, S_y^{\T}]\right\|^2,
\end{align*}
we have
\begin{align*}
\min_{X,Y}\|L_*-XY^{\T}\|^2
&\ge \max\left\{\left\|C_x [C_x, S_x]^{\T}  \begin{bmatrix} \Sigma_1 & 0  \\ 0 & 0 \end{bmatrix}
\begin{bmatrix}C_y\\ S_y\end{bmatrix}S_y^{\T}\right\|^2
,\left\|S_x [C_x, S_x]^{\T}  \begin{bmatrix} \Sigma_1 & 0  \\ 0 & 0 \end{bmatrix}
\begin{bmatrix}C_y\\ S_y\end{bmatrix}C_y^{\T}\right\|^2\right\}\\
&\ge  \max\{\sigma^2_{\min}(C_x)\|S_y\|^2, \sigma_{\min}^2(C_y)\|S_x\|^2\}
\sigma_{\min}^2\left([C_x, S_x]^{\T}  \begin{bmatrix} \Sigma_1 & 0  \\ 0 & 0 \end{bmatrix}
\begin{bmatrix}C_y\\ S_y\end{bmatrix}\right)\\
&\ge \max\{(1-\theta_x^2)\theta_y^2, (1-\theta_y^2)\theta_x^2\} (\sigma_{r*} \sqrt{1-\theta_x^2}\sqrt{1-\theta_y^2} )^2
\end{align*}
Combining it with the fact that $\|L_*-G-XY^{\T}\|\ge \|L_*-XY^{\T}\| -\|G\|$ for any $X$, $Y$,
we get the conclusion.
\end{proof}

\begin{lemma}\label{lem:angle}
Let $L_*$, $G$ be the same as in Lemma~\ref{lem:lb}.
Let $X=(L_*-G)Y$, where $Y\in\R^{n\times r}$ is orthonormal.
Denote $\theta_x=\|\sin\Theta(U_*,X)\|$, $\theta_y=\|\sin\Theta(V_*,Y)\|$.
If $\|G\|< \sigma_{r*}\sqrt{1-\theta_y^2}$, then
\begin{align*}
\sigma_r(X)\ge \sigma_{r*} \sqrt{1- \theta_y^2} - \|G\|,\qquad
 \theta_x \le  \frac{\|G\|}{ \sigma_r \sqrt{1- \theta_y^2} - \|G\|}.
\end{align*}
\end{lemma}

\begin{proof}
By Lemma~\ref{lem:eig} and Lemma~\ref{lem:sin}, we have
\begin{align}
\sigma_r(X)&=\sigma_r((L_*-G)Y)
\ge \sigma_r(L_*Y)-\|GY\|
\ge \sigma_r(\Sigma_*V_*^{\T}Y)-\|G\|
\ge \sigma_{r_*}\sigma_{\min}(V_*^{\T}Y)-\|G\|\notag\\
&= \sigma_{r*}\sigma_{\min}^{\frac12}(Y^{\T}V_*V_*^{\T}Y)-\|G\|
\ge \sigma_{r*}\sigma_{\min}^{\frac{1}{2}}(I_r- Y^{\T}(I-V_*V_*^{\T})Y)-\|G\|\notag\\
&=\sigma_{r*}\sqrt{1- \|(I-V_*V_*^{\T})Y\|^2}-\|G\|
=\sigma_{r*}\sqrt{1- \theta_y^2}-\|G\|
>0.\label{bnd:G}
\end{align}
Therefore, $X$ has full column rank.
Denote $G_x=(X^{\T}X)^{-\frac12}$, $\wht{X}=XG_x$.
Then $\wht{X}$ and $X=AY$ can be rewritten as $\wht{X}=AYG_x$.
Using Lemma~\ref{lem:sin} and \eqref{bnd:G}, we have
\begin{align*}
\|\theta_x\|
&=\|U_{*,c}^{\T}\wht{X}\|
=\|U_{*,c}^{\T}(L_*-G)YG_x\|
\le\|GYG_x\|
\le  \|G\|\|G_x\|
\le\frac{\|G\|}{\sigma_r(X)}\le \frac{\|G\|}{\sigma_{r*}\sqrt{1- \theta_y^2}-\|G\|}.
\end{align*}
The proof is completed.
\end{proof}

\begin{lemma}\label{lem:2inf}
Let $U$, $X\in\R^{m\times r}$ both have orthonormal columns.
It holds $\|X\|_{2,\infty}\le \|U\|_{2,\infty}+\|\sin\Theta(U,X)\|$.
\end{lemma}

\begin{proof}
Let $U_c$ be such that $[U,U_c]$ is an orthogonal matrix.
We can write $X=UC_x+U_c S_x$, where $C_x^{\T}C_x+S_x^{\T}S_x=I_r$.
By Lemma~\ref{lem:sin}, we have $\|\sin\Theta(U,X)\|=\|U_c^{\T}X\|=\|S_x\|$.
Then for any $1\le i\le m$, we have
\begin{align*}
\|e_i^{\T}X\|=\|e_i^{\T}UC_x+e_i^{\T}U_cS_x\|\le \|e_i^{\T}U\|+\|S_x\|,
\end{align*}
the conclusion follows.
\end{proof}

 \begin{lemma}\label{lem:prob}\citep[Lemmas 8,10]{Proc:Jain_COLT15} 
Let $A\in\R^{m\times n}$ with $m\ge n$. Suppose $\Omega$ is obtained by sampling each entry of $A$ with probability $p\in [\frac{1}{4m}, 0.5]$.
Then w.p. $\ge 1-1/m^{10+\log \alpha}$,
\begin{align*}
\|\frac1p \PO(A)-A\|\le \frac{6\sqrt{\alpha m}}{\sqrt{p}} \|A\|_{\max}.
\end{align*}
\end{lemma}

\section{Proof for Main Theorems}

\subsection{Proof of Theorem 1}
\noindent{\bf Proof of Theorem~1.}
    First, it holds
    $\|(I-U_*U_*^{\T})M(I-V_*V_*^{\T})\|=\|(I-U_*U_*^{\T})S_*(I-V_*V_*^{\T})\|$.
    Then by assumption, we have $\|(I-U_*U_*^{\T})M(I-V_*V_*^{\T})\|<{\sigma}_{r*}$.

    Second, we have
    \begin{align*}
    \|E\|&=\|MV_*-U_*\Sigma_*\|
    =\|L_*V_*-U_*\Sigma_*+S_*V_*\|
    =\|S_*V_*\|,\\
    \|F\|&=\|M^{\T}U_*-V_*\Sigma_*\|
    =\|L_*^{\T}U_*-V_*\Sigma_*+S_*^{\T}U_*\|
    =\|S_*^{\T}U_*\|.
    \end{align*}
    It follows
    \[
    \max\{\|E\|,\|F\|\}=\max \{\|S_*V_*\|, \|S_*^{\T}U_*\|\}<\sigma_r-\sigma_{r+1}.
    \]
    Then applying Lemma~\ref{lem:errbound} gives the conclusion.
\qquad \hfill $\square$

\subsection{Proof of Theorem 2}
Throughout the rest of this section, we follow the notations in Algorithm~1.
Besides that, we will also adopt the following notations.
Denote
\begin{align}
r=\rank(L_*),\qquad \kappa_*=\kappa_2(L_*), \qquad p'=p(1-\varrho), \qquad \Omega_t=\Omega/\supp(S_t), \qquad G_t=S_t-S_*.
\end{align}
The SVDs of $L_*$ is given by
\begin{align}
L_*&=[U_*,U_{*,c}]\diag(\Sigma_*,0)[V_*,V_{*,c}]^{\T},
\end{align}
where $[U_*,U_{*,c}]$ and $[V_*,V_{*,c}]$ are orthogonal matrices
$U_*\in\R^{m\times r}$ and $V_*\in\R^{n\times r}$, $\Sigma_*=\diag(\sigma_{1*},\dots,\sigma_{r*})$ with $\sigma_{1*}\ge \dots\ge \sigma_{r*}>0$.
Further denote
\begin{align}
\theta_{x,t}=\|\sin\Theta(U_*,X_t)\|,\qquad
\theta_{y,t}=\|\sin\Theta(V_*,Y_t)\|.
\end{align}

 \begin{lemma}\label{lem:ss}
$\|S_t-S_*\|_{\max}\le 2\|\PO(X_t\Sigma_tY_t^{\T}-L_*)\|_{\max}$ for $t=0,1,\dots$.
\end{lemma}

\begin{proof}
Denote  $\Phi_*=\supp(S_*)$, $\Phi_t=\supp(S_t)$, it is obvious that $S_t-S_*$ is supported on $\Phi_t\cup\Phi_*$
and $\Phi_t\cup\Phi_*\subset \Omega$.
Now we claim that
\[
\|\PO(S_t-S_*)\|_{\max}\le 2\|\PO(X_t\Sigma_tY_t^{\T}-L_*)\|_{\max}.
\]
To show the claim, it suffices to consider the following two cases.

{\bf Case (1)} For any $(i,j)\in \Phi_t$, it holds $(S_t)_{(i,j)}=(L_*+S_*-X_t\Sigma_tY_t^{\T})_{(i,j)}$.
Then it follows that
\begin{align*}
|(S_t-S_*)_{(i,j)}|
=|(L_*-X_t\Sigma_tY_t^{\T})_{(i,j)}|
\le \|\PO(X_t\Sigma_tY_t^{\T}-L_*)\|_{\max}.
\end{align*}

{\bf Case (2)} For any  $(i,j)\in\Phi_* \setminus \Phi_t$, it holds $(S_t)_{(i,j)}=0$.
If $|(S_t-S_*)_{(i,j)}| = |(S_*)_{(i,j)}|>2\|\PO(X_t\Sigma_tY_t^{\T}-L_*)\|_{\max}$, then
\[
|(L_*+S_*-X_t\Sigma_tY_t^{\T})_{(i,j)}|> \|\PO(L_* - X_t\Sigma_tY_t^{\T})\|_{\max}.
\]
Noticing that $S_*$ only changes $s$ entries of $\PO(L_*-X_t\Sigma_tY_t^{\T})$,
we know that the $(i,j)$ entry of $|\PO(L_*+S_*-X_t\Sigma_tY_t^{\T})|$ is larger than
the $(s+1)$st largest entry of $|\PO(L_*+S_*-X_t\Sigma_tY_t^{\T})|$.
This contradicts with $(i,j)\notin\Phi_t$.
\end{proof}



\begin{lemma}\label{lem:s0}
Assume {\bf (A1)}.
Denote $r_s' = \frac{\|S_0-S_*\|_F^2}{\|S_0-S_*\|^2}$.
Let $S_0$ be obtained as in Algorithm~1.
It holds
\begin{align*}
\|S_0-S_*\|\le 2\sqrt{\frac{2\varrho p}{r_s'}} \mu r \|L_*\|.
\end{align*}
\end{lemma}

\begin{proof}

First, for any $i,j$, we have $L_{ij}=e_i^{\T}U_*\Sigma_*V_*^{\T}e_j$.
Using {\bf (A1)}, we have
\[
|L_{ij}| \le \|e_i^{\T}U_*\| \|\Sigma_*\| \|e_j^{\T}V_*\|\le \frac{\mu r}{\sqrt{mn}}\|L_*\|,
\]
and hence
\begin{align}\label{ll}
\|L_*\|_{\max}\le  \frac{\mu r}{\sqrt{mn}}\|L_*\|.
\end{align}
By Lemma~\ref{lem:ss}, we have
\begin{align}\label{s0s1}
\|S_0-S_*\|_{\max}\le 2\|\PO(L_*)\|_{\max}\le \frac{2 \mu r}{\sqrt{mn}}\|L_*\|.
\end{align}
Therefore, using \eqref{ll}, \eqref{s0s1} and {\bf (A2)}, we have
\begin{align}\label{s0s}
\|S_0-S_*\|_F
\le \sqrt{2s} \|S_0-S_*\|_{\max}
\le 2\sqrt{2s} \|\PO(L_*)\|_{\max}
\le 2\sqrt{2s} \|L_*\|_{\max}
\le 2\sqrt{2\varrho p} \mu r \|L_*\|.
\end{align}
By the definition of $r_s'$, it follows that
\begin{align*}
\|S_0-S_*\|&\le\frac{\|S_0-S_*\|_F}{\sqrt{r_s'}}\le 2\sqrt{\frac{2\varrho p}{r_s'}} \mu r \|L_*\|.
\end{align*}
The proof is completed.
\end{proof}

\noindent{\bf Proof of Theorem~2.}
By {\bf (A3)}, Lemma~\ref{lem:prob} and \eqref{ll}, w.p. $\ge 1-1/m^{10+\log \alpha}$, it holds
\begin{align}\label{poll}
\|\frac{1}{p'} \PT{0}(L_*) - L_*\| \le \frac{6\sqrt{\alpha m}}{\sqrt{p'}} \|L_*\|_{\max}
\le \xi \mu r \|L_*\|.
\end{align}

Using Lemma~\ref{lem:s0} and \eqref{poll}, we have w.p. $\ge 1-1/m^{10+\log \alpha}$,
\begin{align}\label{poms}
\|\frac{1}{p'} \PT{0}(M-S_0)  - L_*\|
\le \|\frac{1}{p'}\PT{0}(L_*)-L_*\| + \frac{1}{p'} \|S_0-S_*\|
\le (\xi+\gamma)\mu r\|L_*\|.
\end{align}

Let the SVD of $X_1^{\T}L_*Y_1$ be $X_1^{\T}L_*Y_1=P\wtd{\Sigma}Q^{\T}$,
where $P$, $Q$ are orthogonal matrices, $\wtd{\Sigma}=\diag(\tilde{\sigma}_1,\dots,\tilde{\sigma}_r)$.
Denote $\wtd{U}=X_1P$, $\wtd{V}=Y_1Q$, and let
\begin{align}
E =L_* \wtd{V} - \wtd{U} \wtd{\Sigma},\qquad F= L_*^{\T} \wtd{U} - \wtd{V} \wtd{\Sigma}.
\end{align}
Then it follows that
\begin{align}\label{sigsig}
\|\Sigma_1-X_1^{\T}L_*Y_1\|
=\|X_1^{\T}[\frac{1}{p'} \PT{0}(M-S_0) -L_*]Y_1\|
\le\| \frac{1}{p'} \PT{0}(M-S_0) -L_*\|.
\end{align}

Using \eqref{poms} and \eqref{sigsig}, by calculations, we get
\begin{align}\label{e}
\|E\|&=\|L_* \wtd{V} - \wtd{U} \wtd{\Sigma}\|=\|L_* Y_1 - X_1 P \wtd{\Sigma}Q^{\T}\|
=\|L_*Y_1 - X_1 X_1^{\T}L_*Y_1\|\notag\\
&\le  \|L_*Y_1 - X_1 {\Sigma}_1\|+ \|X_1({\Sigma}_1- X_1^{\T}L_*Y_1)\|\notag\\
& = \|L_*Y_1 - \frac{1}{p'} \PT{0}(M-S_0) Y_1  \| + \|{\Sigma}_1- X_1^{\T}L_*Y_1\|\notag\\
&\le 2\|\frac{1}{p'} \PT{0}(M-S_0)  - L_*\|
\le 2(\xi+\gamma)\mu r\|L_*\|,
\quad \mbox{w.p. } \ge 1-1/m^{10+\log \alpha}.
\end{align}
Similarly, we get
\begin{align}\label{f}
\|F\|\le 2(\xi+\gamma)\mu r\|L_*\|, \quad \mbox{w.p. } \ge 1-1/m^{10+\log \alpha}.
\end{align}

Next, we only need to show $\max\{\|E\|,\|F\|\}\le \sigma_{r*}$ and
$\|(I_m-\wtd{U}\wtd{U}^{\T})L_*(I_n-\wtd{V}\wtd{V}^{\T})\|<\tilde{\sigma}_{r}$.
Once these two inequalities hold, we may apply Lemma~\ref{lem:errbound}.

For the first inequality, using \eqref{e}, \eqref{f} and the assumption $(\xi+\gamma)\mu \kappa r<\frac16$, we get
\begin{align}\label{ef-con}
\max\{\|E\|,\|F\|\}
\le 2( \xi  + \gamma )\mu r\|L_*\| < \sigma_{r*}, \quad \mbox{w.p. } \ge 1-1/m^{10+\log \alpha}.
\end{align}

For the second inequality, using \eqref{s0s} and \eqref{poll}, we have
\begin{align}\label{pomsl}
\|\frac{1}{p'}\PO(M-S_0)-L_*\|\le \|\frac{1}{p'}\PO(L_*)-L_*\|+\frac{1}{p'}\|S_*-S_0\|\le (\xi+\gamma)\mu r \|L_*\|.
\end{align}
Then using Lemma~\ref{lem:eig}, \eqref{poms},\eqref{sigsig} and \eqref{pomsl}, we have
\begin{align*}
|\tilde{\sigma}_r - \sigma_{r*}|
\le |\tilde{\sigma}_r - \hat{\gamma}_{r,0}|+|\hat{\gamma}_{r,0} - \sigma_{r*}|
\le \|X_1^{\T}L_*Y_1 - {\Sigma}_1\| +\|\frac{1}{p'} \PT{0}(M-S_0)  - L_*\|
 \le 2(\xi + \gamma)\mu r \|L_*\|.
\end{align*}
It follows that
\begin{align}\label{sigr2}
\tilde{\sigma}_r\ge \sigma_{r*} - 2(\xi + \gamma)\mu r \|L_*\|.
\end{align}

Then using the assumption $(\xi+\gamma)\mu \kappa r<\frac16$, \eqref{poms} and \eqref{sigr2}, we have
\begin{align*}
&\|(I_m-\wtd{U}\wtd{U}^{\T})L_*(I_n-\wtd{V}\wtd{V}^{\T})\|
=\|(I_m-X_1X_1^{\T})[L_*- \frac{1}{p'} \PT{0}(M-S_0)](I_n-Y_1Y_1^{\T})\|\\
\le &\|L_*- \frac{1}{p'} \PT{0}(M-S_0)\|
\le (\xi+\gamma)\mu r \|L_*\|
<\sigma_{r*} - 2(\xi + \gamma)\mu r \|L_*\|
\le \tilde{\sigma}_{r}.
\end{align*}

Now using \eqref{e}, \eqref{f}, the assumption $(\xi+\gamma)\mu \kappa r<\frac16$ and Lemma~\ref{lem:errbound},
we have
\begin{subequations}
\begin{align}
&\max\{\theta_{x,1},\theta_{y,1}\}
=\max\{\|\sin\Theta(U_*,\wtd{U})\|,\|\sin\Theta(V_*,\wtd{V})\|\}
\le \frac{2(\xi+\gamma)\mu r \kappa}{1-1/3}=3(\xi+\gamma)\mu r \kappa,\label{thetaxy1}\\
&\|L_*-\wtd{U}\wtd{\Sigma}\wtd{V}^{\T}\|_{\max}/\|L_*\|
\le (\|U_*\|_{2,\infty}\theta_{y,1} + \|V_*\|_{2,\infty}\theta_{x,1}) +  (1+3\|U_*\|_{2,\infty} \|V_*\|_{2,\infty})\theta_{x,1}\theta_{y,1}.\label{lusv}
\end{align}
\end{subequations}

Using the assumption $(\xi+\gamma)\mu \kappa r<\frac13\sqrt{\frac{\mu_1'r}{m}}$, by \eqref{thetaxy1}, we have $\max\{\theta_{x,1},\theta_{y,1}\}\le\sqrt{\frac{\mu_1' r}{m}}$.
On the other hand, assumption {\bf (A1)} implies that
\begin{align}\label{uvinf}
\|U_*\|_{2,\infty}\le \sqrt{\frac{\mu r}{m}}, \qquad \|V_*\|_{2,\infty}\le \sqrt{\frac{\mu r}{n}}.
\end{align}
Then it follows from Lemma~\ref{lem:2inf} that
\begin{subequations}\label{x1y1}
\begin{align}
\|X_1\|_{2,\infty}&\le \|U_*\|_{2,\infty}+\|\sin\Theta(X_1,U_*)\|\le \sqrt{\frac{\mu r}{m}}+ \sqrt{\frac{\mu_1' r}{m}}\le \sqrt{\frac{\mu_1 r}{m}},\\
\|Y_1\|_{2,\infty}&\le \|V_*\|_{2,\infty}+\|\sin\Theta(Y_1,V_*)\|\le \sqrt{\frac{\mu r}{n}}+ \sqrt{\frac{\mu_1' r}{m}}\le \sqrt{\frac{\mu_1 r}{n}}.
\end{align}
\end{subequations}

Using the assumption $(\xi+\gamma)\mu \kappa r<\frac13\sqrt{\frac{\mu_1'r}{m}}$, \eqref{poms}, \eqref{sigsig}, \eqref{lusv}, \eqref{uvinf} and \eqref{x1y1}, by calculations, we have
\begin{align*}
\|L_*-X_1\Sigma_1Y_1^{\T}\|_{\max}/\|L_*\|
\le &\|L_*-\wtd{U}\wtd{\Sigma}\wtd{V}^{\T}\|_{\max}/\|L_*\| + \|\wtd{U}\wtd{\Sigma}\wtd{V}^{\T}-X_1\Sigma_1Y_1^{\T}\|_{\max}/\|L_*\|\\
= &\|L_*-\wtd{U}\wtd{\Sigma}\wtd{V}^{\T}\|_{\max}/\|L_*\| + \|X_1(X_1^{\T}L_*Y_1 -\Sigma_1)Y_1^{\T}\|_{\max}/\|L_*\|\\
\le & \|L_*-\wtd{U}\wtd{\Sigma}\wtd{V}^{\T}\|_{\max}/\|L_*\| + \|X_1\|_{2,\infty}\|X_1^{\T}L_*Y_1 -\Sigma_1\| \|Y_1\|_{2,\infty}/\|L_*\|\\
\le &\|L_*-\wtd{U}\wtd{\Sigma}\wtd{V}^{\T}\|_{\max}/\|L_*\| + \|X_1\|_{2,\infty}\|Y_1\|_{2,\infty} (\xi+\gamma)\mu r \kappa,\\
\le & (\|U_*\|_{2,\infty}\theta_{y,1} + \|V_*\|_{2,\infty}\theta_{x,1}) +  (1+3\|U_*\|_{2,\infty} \|V_*\|_{2,\infty})\theta_{x,1}\theta_{y,1} \\
&\mbox{}\hskip1.72in+ \|X_1\|_{2,\infty}\|Y_1\|_{2,\infty}\frac13\sqrt{\frac{\mu_1' r}{m}}\\
\le & \left( \sqrt{\frac{\mu r}{m}} \theta_{y,1} + \sqrt{\frac{\mu r}{n}} \theta_{x,1} + \theta_{x,1}\theta_{y,1}\right) +\mathcal{O}(n^{-3/2}),
\end{align*}
which completes the proof.
\hfill $\square$

\subsection{Proof of Theorem 3}
\noindent{\bf Proof of Theorem~3.}
First, we give an upper bound for $\sup_{X\in\R^{m\times r}}\|\PT{t}(R)  \PT{t}(XY_t^{\T})^{\T}\|/\|X\|$.
Let $\{\delta_{ij}\}$ be an independent family of {\sc Bernoulli}($p'$) random variables, $X^{\T}=[x_1,\dots,x_m]\in\R^{r\times m}$ be arbitrary nonzero matrix with $\|X\|=1$, and $Y_t^{\T}=[y_1,\dots,y_n]$.
Denote $E_{ij}=e_ie_j^{\T}$, $R=[r_{ij}]$, $\bw_{il}= \sum_{j,k}\delta_{ij}r_{ij}E_{ij} \delta_{lk} x_k^{\T}y_l E_{kl}^{\T}$,
$\bz=\sum_{i,l} \bz_{i,l}$.
By calculations, we have
\begin{align*}
&\mathbb{E}(\bw_{il})=p'^2 \sum_{j,k}r_{ij} E_{ij} y_k^{\T}x_l E_{kl},
=p'^2 \sum_{j} r_{ij} E_{ij} y_j^{\T}x_l E_{jl}=p'^2 R_{(i,:)} Y_tx_l E_{il}=0, \\
&\|\bw_{il}\|\le \sqrt{p'n} \max|r_{ij}x_j^{\T}y_l| \le \sqrt{\mu' r p'} \|R\|_{\max},\\
&\|\mathbb{E}[\sum_{i,l}\bw_{il}\bw_{il}^{\T}]\|
=\|\mathbb{E}[\sum_{i,l}(\sum_{j,k}\delta_{ij}r_{ij}E_{ij} \delta_{lk} x_k^{\T}y_l E_{kl}^{\T})(\sum_{j',k'}\delta_{ij'}r_{ij'}E_{ij'} \delta_{lk'} x_{k'}^{\T}y_l E_{k'l}^{\T})^{\T}]\|=0,\\
&\|\mathbb{E}[\sum_{i,l}\bw_{il}^{\T}\bw_{il}]\|=0.
\end{align*}

By Lemma~\ref{lem:bernstein}, we have
$\mathbb{P}\{\|\bw\|>t\} \le (m+n)\exp\Big(- \frac{3t/2}{\sqrt{\mu' r p'}\|R\|_{\max}} \Big)$.
Let $t=\frac{2}{3}(\log(m+n)+5)\sqrt{\mu' r p'}\|R\|_{\max}$, then w.p. $\ge 0.99$, it holds
\begin{align}\label{ubw}
\|\bw\|\le\frac{2}{3}(\log(m+n)+5)\sqrt{\mu' r p'}\|R\|_{\max}.
\end{align}

Second, It is easy to see that $X_{\opt}=(M-S_t)Y_t$.
By calculations, we have
\begin{align}
&\min_X\|\PT{t}(XY_t^{\T}) - \PT{t}(M-S_t)\|^2 \label{xyt}\\
=& \min_{\Delta X} \|\PT{t}((X_{\opt}+\Delta X)Y_t^{\T}) - \PT{t}\big((M-S_t)Y_tY_t^{\T} + (M-S_t)(I-Y_tY_t^{\T})\big)\|^2\notag\\
=& \min_{\Delta X} \|\PT{t}(\Delta XY_t^{\T}) - \PT{t}(R)\|^2.\label{deltax}
\end{align}
Then we declare that \eqref{deltax} is minimized when $\Delta X=\wtd{X}_{\opt}-X_{\opt}$.
This is because \eqref{xyt} is minimized when $X=\wtd{X}_{\opt}$ and $X=X_{\opt}+\Delta X$.
Thus, we have
\begin{align}\label{xdx}
\|\wtd{X}_{\opt}-X_{\opt}\|=\|\Delta X\|\le \frac{\sup_{X\in\R^{m\times r}}\|\PT{t}(R)  \PT{t}(XY_t^{\T})^{\T}\|}{\sigma^2}.
\end{align}
Substituting \eqref{ubw} into \eqref{xdx}, we get the conclusion.
\hfill $\square$

\subsection{Proof of Theorem~4}
\begin{lemma}\label{lem:st}
Denote $r_s = \inf_t\frac{\|S_t-S_*\|_F^2}{\|S_t-S_*\|^2}$, $\zeta=\sqrt{\frac{2s\mu r}{m r_s}}$.
If $\|L_*-X_t\Sigma_tY_t^{\T}\|_{\max}\le c_t\|L_*\|\sqrt{\frac{\mu r}{m}}$ for some positive parameter $c_t$,
then
\begin{align*}
\|S_t-S_*\|\le 2 c_t \|L_*\| \zeta,
\qquad |\gamma_{j,t}-\sigma_{j*}|\le 2 c_t \|L_*\|\zeta.
\end{align*}
\end{lemma}

\begin{proof}
Using Lemma~\ref{lem:ss}, by simple calculations, we have
\begin{align*}
\|S_t-S_*\|&\le\frac{\|S_t-S_*\|_F}{\sqrt{r_s}}
\le \sqrt{\frac{2s}{r_s}} \|S_t-S_*\|_{\max}
\le 2\sqrt{\frac{2s}{r_s}} \|\PO(L_*-X_t\Sigma_t Y_t^{\T})\|_{\max}
\le 2 c_t \|L_*\| \sqrt{\frac{2s\mu r}{m r_s}}
=2 c_t  \|L_*\|\zeta.
\end{align*}
 Then by Lemma~\ref{lem:eig}, we know that
 \[
 |\gamma_{j,t}-\sigma_{j*}|\le \|(M-S_t)-L_*\|=\|S_t-S_*\|\le 2 c_t \|L_*\| \zeta.
 \]
The proof is completed.
\end{proof}

\noindent{\bf Proof of Theorem~4.}
First, denote $\bar{X}_{t+1}=(M-S_{t})Y_{t}$, then we know that $\bar{X}_{t+1}$ is the solution to
$\min_X\|M-S_{t}-XY_{t}^{\T}\|$.
Also note that $\wtd{X}_{t+1}$ on line 8 of Algorithm~1 is the solution to $\min_X\|\PT{t}(M-S_{t}-XY_{t}^{\T})\|$.
Then by Theorem~3, we have
\begin{align*}
\|\bar{X}_{t+1}-\wtd{X}_{t+1}\|\le C_{\rm LS}\|(M-S_{t})(I-Y_tY_{t}^{\T})\|_{\max},\quad \mbox{w.p. } \ge 0.99.
\end{align*}
Then it follows that from Lemma~\ref{lem:sin}, Lemma~\ref{lem:ss} and Lemma~\ref{lem:st} that
\begin{align}
\|\bar{X}_{t+1}-\wtd{X}_{t+1}\|
&\le C_{\rm LS}(\|L_*(I-Y_tY_t^{\T})\|_{\max}+\|(S_t-S_*)(I-Y_tY_t^{\T})\|_{\max})\notag\\
&\le C_{\rm LS}(\|L_*\|\sqrt{\frac{\mu r}{m}} \theta_{y,t} + \|S_t-S_*\|_{2,\infty})
\le C_{\rm LS}(\|L_*\|\sqrt{\frac{\mu r}{m}} \theta_{y,t} + \sqrt{2p\varrho n}\|S_t-S_*\|_{\max})\notag\\
&\le C_{\rm LS}(\|L_*\|\sqrt{\frac{\mu r}{m}} \theta_{y,t} + 2\sqrt{2p\varrho n}\|L_*-X_t\Sigma_tY_t^{\T}\|_{\max})
\le \frac{C}{\sqrt{m}} \|L_*\|\theta_{y,t}.
\label{xwtdx}
\end{align}


Second, using Lemma~\ref{lem:st} and $4c\kappa\zeta<1$, we have
 \begin{align}\label{stss}
 \|S_t-S_*\|< 2c \theta_{y,t}\|L_*\|\zeta\le \sqrt{2}c\|L_*\|\zeta < \frac{\sigma_{r*}}{\sqrt{2}}\le\sigma_{r*}\sqrt{1-\theta_{y,t}^2},
 \end{align}
Then by Lemma~\ref{lem:angle}, we know that
\begin{align}\label{xt1u}
\|\sin\Theta(\bar{X}_{t+1},U_*)\|
\le \frac{\|S_{t}-S_*\|}{ \sigma_{r*}\sqrt{1- \theta_{y,t}^2} - \|S_t-S_*\|}.
\end{align}
Using \eqref{xt1u}, the assumption $\|L_*-X_t\Sigma_tY_t^{\T}\|_{\max}\le c \|L_*\|\; \theta_{y,t} \sqrt{\frac{\mu r}{m}}$, Lemma~\ref{lem:st} and $\theta_{y,t}\le \frac{1}{\sqrt{2}}$, we get
\begin{align}\label{xuf}
\|\sin\Theta(\bar{X}_{t+1},U_*)\|
\le  \frac{2c \|L_*\|\zeta \theta_{y,t}}{\frac{\sigma_{r*}}{\sqrt{2}} - 2c \|L_*\|\zeta \theta_{y,t}}
\le \frac{2\sqrt{2} c \kappa\zeta \theta_{y,t}}{1 - 2c \kappa\zeta}<4\sqrt{2}c\kappa\zeta\theta_{y,t}.
\end{align}

Therefore, using Lemma~\ref{lem:sin},  \eqref{xwtdx} and \eqref{xuf}, we have
\begin{align}
\|\theta_{x,t+1}\| &= \|U_{*,c}^{\T}X_{t+1}\|= \|U_{*,c}^{\T}\wtd{X}_{t+1}R_{x,t+1}^{-1}\|
\le  \|U_{*,c}^{\T}\bar{X}_{t+1}R_{x,t+1}^{-1}\| +\|U_{*,c}^{\T}(\wtd{X}_{t+1}-\bar{X}_{t+1})R_{x,t+1}^{-1}\| \notag\\
&\le \|R_{x,t+1}^{-1}\| (\|\sin\Theta(\bar{X}_{t+1},U_*)\|\|\bar{X}_{t+1}\| + \|\wtd{X}_t-\bar{X}_t\|)\notag\\
&\le \frac{1}{\sigma_r(\wtd{X}_{t+1})} \Big( 4\sqrt{2}c\kappa\zeta \|\bar{X}_{t+1}\| + \frac{C}{\sqrt{m}} \|L_*\|\Big) \theta_{y,t}.\label{thetaxt1}
\end{align}

Now using Lemma~\ref{lem:eig}, \eqref{xwtdx}, $\theta_{y,t}\le \frac{1}{\sqrt{2}}$ and \eqref{stss}, we have
\begin{align*}
\|\bar{X}_{t+1}\|&=\|(M-S_t)Y_t\|\le \|L_*Y\|+\|S_t-S_*\|\le\|L_*\|+\sqrt{2}c\|L_*\|\zeta,\\
\sigma_r(\wtd{X}_{t+1})
&\ge\sigma_r(\bar{X}_{t+1})-\frac{C}{\sqrt{m}}\|L_*\|\theta_{y,t}
\ge \sigma_r((M-S_t)Y_t)-\frac{C}{\sqrt{2m}}\|L_*\|
\ge \sigma_r(L_*Y_t) -\|S_t-S_*\|-\frac{C}{\sqrt{2m}}\|L_*\|\\
&\ge \sigma_{r*}\sqrt{1-\theta_{y,t}^2} -\sqrt{2}c\|L_*\|\zeta-\frac{C}{\sqrt{2m}}\|L_*\|
\ge \frac{\sigma_{r*}}{\sqrt{2}} -\sqrt{2}c\|L_*\|\zeta-\frac{C}{\sqrt{2m}}\|L_*\|.
\end{align*}
Substituting them into \eqref{thetaxt1}, we get the conclusion.
\hfill $\square$

\subsection{Proof of Theorem~5}
\begin{lemma}\label{lem:lxymax}
Follow the notations and assumptions in Lemma~1.
Then
\begin{align*}
\|L_*-\wht{X}_{t+1}R_{x,t+1}Y_t^{\T}\|_{\max}\le
\Big((1+C_{\rm LS} \sqrt{\frac{\mu' r}{n}}) \sqrt{\frac{\mu r}{m}} + (1+C_{\rm LS} \sqrt{2\varrho p n})\sqrt{\frac{\mu' r}{n}} 2c\zeta\Big)\|L_*\|\theta_{y,t}.
\end{align*}
\end{lemma}

\begin{proof}
Direct calculations give rise to
\begin{subequations}
\begin{align}
\|L_*-\wtd{X}_{t+1}Y_t^{\T}\|_{\max}
&\le \|L_*-(M-S_t)Y_tY_t^{\T}\|_{\max}+\|(M-S_t)Y_tY_t^{\T}-\wtd{X}_{t+1}Y_t^{\T}\|_{\max}\notag\\
&\le  \|L_*-(M-S_t)Y_tY_t^{\T}\|_{\max}+ \|(M-S_t)Y_t-\wtd{X}_{t+1}\| \sqrt{\frac{\mu' r}{n}}\label{eq1}\\
&\le \|L_*-(M-S_t)Y_tY_t^{\T}\|_{\max} +C_{\rm LS} \sqrt{\frac{\mu' r}{n}}\|(M-S_t)(I-Y_tY_t^{\T})\|_{\max}\label{eq2}\\
&\le (1+C_{\rm LS} \sqrt{\frac{\mu' r}{n}})\|L_*(I-Y_tY_t^{\T})\|_{\max}+ (1+C_{\rm LS} \sqrt{2\varrho p n})\sqrt{\frac{\mu' r}{n}}\|S_t-S_*\|_{\max}\notag\\
&\le \Big((1+C_{\rm LS} \sqrt{\frac{\mu' r}{n}}) \sqrt{\frac{\mu r}{m}} + (1+C_{\rm LS} \sqrt{2\varrho p n})\sqrt{\frac{\mu' r}{n}} 2c\zeta\Big)\|L_*\|\theta_{y,t}\label{eq3}
\end{align}
\end{subequations}
where \eqref{eq1} uses $\|Y_t\|_{2,\infty}\le \sqrt{\frac{\mu' r}{m}}$,
\eqref{eq2} uses Theorem~3,
\eqref{eq3} uses the SVD of $L_*$, $\|U_*\|_{2,\infty}\le \sqrt{\frac{\mu r}{m}}$ and Lemme~\ref{lem:sin}.
\end{proof}

\noindent{\bf Proof of Theorem~5.}
First, by Lemma~\ref{lem:lb}, we have
\begin{align*}
\|M-S_t - X_t\Sigma_tY_t^{\T}\|
&\ge \sigma_{r*} \max\{\sqrt{1-\theta_{x,t}^2}\theta_{y,t}, \sqrt{1-\theta_{y,t}^2}\theta_{x,t}\} \sqrt{1-\theta_{x,t}^2}\sqrt{1-\theta_{y,t}^2} - \|S_t-S_*\|
\end{align*}
Then using \eqref{stss}, $\theta_{x,t}\le \frac{1}{\sqrt{2}}$ and $\theta_{y,t}\le \frac{1}{\sqrt{2}}$, we get
\begin{align}\label{mxxy}
\|M-S_t - X_t\Sigma_tY_t^{\T}\|
&\ge \frac{\sigma_{r*}}{2\sqrt{2}}\theta_{y,t} - 2 c \|L_*\|\;  \sqrt{\frac{\mu r}{m}} \theta_{y,t}.
\end{align}

Second, by calculations, we have
\begin{align}
\|M-S_t-\wht{X}_{t+1}\wtd{Y}_{t+1}^{\T}\|
&\le \|(I-\wht{X}_{t+1}\wht{X}_{t+1}^{\T})(M-S_t)\| + \|\wht{X}_{t+1}\wht{X}_{t+1}^{\T}(M-S_t) - \wht{X}_{t+1}\wtd{Y}_{t+1}^{\T}\|\notag\\
&\le \|(I-\wht{X}_{t+1}\wht{X}_{t+1}^{\T})L_*\| +\|S_t-S_*\|+ \|\wht{X}_{t+1}^{\T}(M-S_t) - \wtd{Y}_{t+1}^{\T}\|\notag\\
&\le   \|L_*\|\theta_{x,t+1} +  2 c \|L_*\|\zeta  \sqrt{\frac{\mu r}{m}} \theta_{x,t+1} +   \frac{C}{\sqrt{m}} \|L_*\|\theta_{x,t+1}\label{eqr2}\\
&\le (1 +  2 c \zeta \sqrt{\frac{\mu r}{m}}  +   \frac{C}{\sqrt{m}}) \phi \|L_*\|\theta_{y,t},\label{eqr3}
\end{align}
where the first two terms of \eqref{eqr2} use Lemma~\ref{lem:sin} and \eqref{stss}, respectively, and the last term can obtained similar to \eqref{xwtdx},
with the help of Lemma~\ref{lem:lxymax}.

Then it follows that
\begin{subequations}
\begin{align}
\|M-S_{t+1}-X_{t+1}\Sigma_{t+1}Y_{t+1}^{\T}\|
&\le \|M-S_t-X_{t+1}\Sigma_{t+1}Y_{t+1}^{\T}\|\label{eqm1}\\
&\le (1 +  2 c  \sqrt{\frac{\mu r}{m}}  +   \frac{C}{\sqrt{m}}) \phi \|L_*\|\theta_{y,t}\label{eqm2}\\
&\le \frac{(1 +  2 c \zeta \sqrt{\frac{\mu r}{m}}  +   \frac{C}{\sqrt{m}}) \phi \|L_*\|}{\frac{\sigma_{r*}}{2\sqrt{2}} - 2 c \zeta \|L_*\|\;  \sqrt{\frac{\mu r}{m}}} \|M-S_t - X_t\Sigma_tY_t^{\T}\|\label{eqm3}\\
&= \psi\|M-S_t - X_t\Sigma_tY_t^{\T}\|,\notag
\end{align}
\end{subequations}
where \eqref{eqm1} uses Lemma~\ref{lem:2fro},
\eqref{eqm2} uses \eqref{eqr3},
\eqref{eqm3} uses \eqref{mxxy}.
The proof is completed.
\hfill $\square$

\end{document}